\documentclass{article}

\usepackage[margin=1in]{geometry}

\usepackage[utf8]{inputenc} 
\usepackage[T1]{fontenc}    
\usepackage{hyperref}       
\usepackage{xurl}            
\usepackage{booktabs}       
\usepackage{amsfonts}       
\usepackage{nicefrac}       
\usepackage{microtype}      

\usepackage{amsmath,amssymb,amsfonts,amsthm}
\usepackage{enumitem}
\usepackage{subcaption}
\usepackage{graphicx}
\usepackage{xcolor}
\usepackage{comment}
\usepackage{algorithmicx}
\usepackage{algpseudocode}
\usepackage{algorithm}
\usepackage{cite}

\setlist[itemize]{leftmargin=*}
\setlist[enumerate]{leftmargin=*}

\newtheorem{thm}{Theorem}

\newtheorem{cor}{Corollary}
\newtheorem{defn}{Definition}
\newtheorem{lemma}{Lemma}

\newcommand{\adrev}{r}
\newcommand{\adpap}{p}
\newcommand{\numrev}{n}
\newcommand{\numpap}{d}
\newcommand{\revset}{\mathcal{R}}
\newcommand{\papset}{\mathcal{P}}
\newcommand{\revload}{k}
\newcommand{\papload}{\ell}
\newcommand{\simmat}{S}
\newcommand{\intas}{M}
\newcommand{\fracas}{F}
\newcommand{\maxprob}{Q}
\newcommand{\badprob}{W}
\newcommand{\maxbadprob}{\lambda}
\newcommand{\maxprobtens}{T}
\newcommand{\fracastens}{G}
\newcommand{\intastens}{M}
\newcommand{\adinst}{I}
\newcommand{\numinst}{m}
\newcommand{\maxprobconst}{q_0}
\newcommand{\groupsize}{g}
\newcommand{\capac}{h}
\newcommand{\muceil}{\lceil \mu \rceil}
\newcommand{\mufloor}{\lfloor \mu \rfloor}
\newcommand{\bidscale}{\gamma}

\DeclareMathOperator*{\argmax}{arg\,max}

\title{Mitigating Manipulation in Peer Review \\ via Randomized Reviewer Assignments}

\author{%
    Steven Jecmen \\
    Carnegie Mellon University \\
    \texttt{sjecmen@cs.cmu.edu} \\
    \and
    Hanrui Zhang \\
    Duke University \\
    \texttt{hrzhang@cs.duke.edu} \\
    \and
    Ryan Liu \\
    Carnegie Mellon University \\
    \texttt{ryanliu@andrew.cmu.edu} \\
    \and
    Nihar B. Shah \\
    Carnegie Mellon University \\
    \texttt{nihars@cs.cmu.edu} \\
    \and
    Vincent Conitzer \\
    Duke University \\
    \texttt{conitzer@cs.duke.edu} \\
    \and   
    Fei Fang \\
    Carnegie Mellon University \\
    \texttt{feif@cs.cmu.edu} \\
}
\date{} 

\begin{document}

\maketitle

\begin{abstract}
We consider three important challenges in conference peer review: (i) reviewers maliciously attempting to get assigned to certain papers to provide positive reviews, possibly as part of quid-pro-quo arrangements with the authors; (ii) ``torpedo reviewing,'' where reviewers deliberately attempt to get assigned to certain papers that they dislike in order to reject them; (iii) reviewer de-anonymization on release of the similarities and the reviewer-assignment code. On the conceptual front, we identify connections between these three problems and present a framework that brings all these challenges under a common umbrella. We then present a (randomized) algorithm for reviewer assignment that can optimally solve the reviewer-assignment problem under any given constraints on the probability of assignment for any reviewer-paper pair. We further consider the problem of restricting the joint probability that certain suspect pairs of reviewers are assigned to certain papers, and show that this problem is NP-hard for arbitrary constraints on these joint probabilities but efficiently solvable for a practical special case. Finally, we experimentally evaluate our algorithms on datasets from past conferences, where we observe that they can limit the chance that any malicious reviewer gets assigned to their desired paper to $50\%$ while producing assignments with over $90\%$ of the total optimal similarity. Our algorithms still achieve this similarity while also preventing reviewers with close associations from being assigned to the same paper. 
\end{abstract}

\section{Introduction}

Peer review, the evaluation of work by others working in the same field as the producer of the work or with similar competencies, is a critical component of scientific research. It is regarded favorably by a significant majority of researchers and is seen as being essential to both improving the quality of published research and validating the legitimacy of research publications \cite{mulligan2013peer, nicholas2015peer, ware2008peer}. Due to the wide adoption of peer review in the publication process in academia, the peer-review process can be very high-stakes for authors, and the integrity of the process can significantly influence the careers of the authors (especially due to the prominence of a ``rich get richer'' effect in academia \cite{merton1968matthew}). 

However, there are several challenges that arise in peer review relating to the integrity of the review process. In this work, we address three such challenges for peer review in academic conferences where a number of papers need to be assigned to reviewers at the same time.

\paragraph{(1) Untruthful favorable reviews.} In order to achieve a good reviewer assignment, peer review systems must solicit some information about their reviewers' knowledge and interests. This inherently presents opportunities for manipulation, since reviewers can lie about their interests and expertise. For example, reviewers often are expected to bid on the papers they are interested in reviewing before an assignment algorithm is run to determine the paper assignment.
This system can be manipulated, and in fact this is known to have happened in at least one ACM conference~\cite{Vijaykumar_2020}:
\begin{quote}
``{\it Another SIG community has had a collusion problem where the investigators found that a group of PC members and authors colluded to bid and push for each other’s papers violating the usual conflict-of-interest rules.}''
\end{quote}
The problem of manipulation is not limited to the bidding system, as practically anything used to determine paper assignments (e.g., self-reported area of expertise, list of papers the reviewer has published) can potentially be manipulated; in more extreme cases, authors have been known to entirely falsify reviewer identities to get a desired reviewer assignment \cite{ferguson2014publishing, gao2017retractions}. 
In some cases, unethical authors may enter into deals with potential reviewers for their paper, where the reviewer agrees to attempt to get assigned to the author's paper and give it a favorable review in exchange for some outside reward (e.g., as part of a quid-pro-quo arrangement for the reviewer's own paper in another publication venue). 
To preserve the integrity of the reviewing process and maintain community trust, the paper assignment algorithm should guarantee the mitigation of these kinds of arrangements. 
    
\paragraph{(2) Torpedo reviewing.} In ``torpedo reviewing,'' unethical reviewers attempt to get assigned to papers they dislike with the intent of giving them an overly negative review and blocking the paper from publication. 
This can have wide-reaching consequences \cite{langford_2008}: 
\begin{quote}
``{\it If a research direction is controversial in the sense that just 2-or-3 out of hundreds of reviewers object to it, those 2 or 3 people can bid for the paper, give it terrible reviews, and prevent publication. Repeated indefinitely, this gives the power to kill off new lines of research to the 2 or 3 most close-minded members of a community, potentially substantially retarding progress for the community as a whole.}''
\end{quote} 
One special case of torpedo reviewing has been called ``rational cheating,'' referring to reviewers negatively reviewing papers that compete with their own authored work \cite{barroga2014safeguarding, paolucci2014mechanism}. The high-stakes atmosphere of academic publishing can exacerbate this problem \cite{akst2010hate}: 
\begin{quote}
``{\it The cutthroat attitude that pervades the system results in ludicrous rejections for personal reasons---if the reviewer feels that the paper threatens his or her own research or contradicts his or her beliefs, for example.}''
\end{quote}
A paper assignment algorithm should guarantee to authors that their papers are unlikely to have been torpedo-reviewed.

\paragraph{(3) Reviewer de-anonymization in releasing assignment data.} For transparency and research purposes, conferences may wish to release the paper-reviewer similarities and the paper assignment algorithm used after the conference. However, if the assignment algorithm is deterministic, this would allow for authors to fully determine who reviewed their paper, breaking the anonymity of the reviewing process. 
Even when reviewer and paper names are removed, identities can still be discovered (as in the case of the Netflix Prize dataset \cite{narayanan2006break}). Consequently, a rigorous guarantee of anonymity is needed in order to release the data. 


~\\ Although these challenges may seem disparate, we address all of them under a common umbrella  framework. Our contributions are as follows:
\begin{itemize} 
    \item {\bf Conceptual:} We formulate problems concerning the three aforementioned issues in peer review, and propose a framework to address them through the use of randomized paper assignments (Section~\ref{secproblem}). 
    \item {\bf Theoretical:} We design computationally efficient, randomized assignment algorithms that optimally assign reviewers to papers subject to given restrictions on the probability of assigning any particular reviewer-paper pair (Section~\ref{secbasicalgo}). We further consider the more complex case of preventing suspicious {\em pairs} of reviewers from being assigned to the same paper (Section~\ref{secinstalgo}). 
    We show that finding the optimal assignment subject to arbitrary constraints on the probabilities of reviewer-reviewer-paper assignments is NP-hard. In the practical special case where the program chairs want to prevent pairs of reviewers within the same subset of some partition of the reviewer set (for example, reviewers at the same academic institution or with the same geographical area of residence) 
    from being assigned to the same paper, we present an algorithm that finds the optimal randomized assignment with this guarantee.
    \item {\bf Empirical:} We test our algorithms on datasets from past conferences and show their practical effectiveness (Section~\ref{secexps}). 
    As a representative example, on data reconstructed from ICLR 2018, our algorithms can limit the chance of any reviewer-paper assignment to $50\%$ while achieving $90.8\%$ of the optimal total similarity. Our algorithms can continue to achieve this similarity while also preventing reviewers with close associations from being assigned to the same paper. We further demonstrate, using the ICLR 2018 dataset, that our algorithm successfully prevents manipulation of the assignment by a simulated malicious reviewer.  
\end{itemize}

All of the code for our algorithms and our empirical results is freely available online.\footnote{\url{https://github.com/theryanl/mitigating_manipulation_via_randomized_reviewer_assignment/}}

\section{Related Literature}

Many paper assignment algorithms for conference peer review have been proposed in past work. The widely-used Toronto Paper Matching System (TPMS) \cite{charlin2013toronto} computes a similarity score for each reviewer-paper pair based on analysis of the reviewers' past work and bids, and then aims to maximize the total similarity of the resulting assignment. 
The framework of ``compute similarities and maximize total similarity'' (and similar variants) encompasses many paper assignment algorithms, where similarities can be computed in various ways from automated and manual analysis and reviewer bids \cite{charlin2012framework, long2013good, goldsmith2007ai, tang2010expertise, flach2010novel,lian2018conference,kobren19localfairness}. We treat the bidding process and computation of similarities as given, and focus primarily on adjusting the optimization problem to address the three aforementioned challenges. Some work has considered other optimization objectives such as fairness \cite{Garg2010papers,stelmakh2018peerreview4all}. We also consider a similar fairness objective in a variant of our algorithm. On a related front, there are also a number of recent works~\cite{ge13bias,noothigattu2018choosing,wang2018your,fiez2020super,roos2011calibrate,stelmakh2019testing,nips14experiment, shah2018design, tomkins17wsdm,kang18peerread,manzoor2020uncovering,stelmakh2020resubmissions,stelmakh2020catch} which deal with various other aspects of peer review.

Much prior work has studied the issue of preventing or mitigating strategic behavior in peer review. This work usually focuses on the incentives reviewers have to give poor reviews to other papers in the hopes of increasing their own paper's chances of acceptance \cite{xu2018strategyproof, aziz2019strategyproof, kurokawa2015impartial, kahng2018ranking, holzman2013impartial}. 
Unlike the issues we deal with in this paper, these works consider only reviewers' incentives to get \emph{their own paper accepted} and not other possible incentives. 
We instead consider arbitrary incentives for a reviewer to give an untruthful review, such as a personal dislike for a research area or misincentives brought about by author-reviewer collusion. Instead of aiming to remove reviewer incentives to write untruthful reviews, our work focuses on mitigating the effectiveness of manipulating the reviewer assignment process. 


A concurrent work~\cite{ding2020privacy} considers a different set of problems in releasing data in peer review while preserving reviewer anonymity. The data to be released here are some function of the scores and the reviewer-assignment, whereas we look to release the similarities and the assignment code. Moreover, the approach and techniques in~\cite{ding2020privacy} are markedly different---they consider post-processing the data for release using techniques such as differential privacy, whereas we consider randomizing the assignment for plausible deniability. 

Outside of peer review, there is a line of prior work focused on detecting and mitigating manipulation in online reviews (such as those on Yelp or Amazon). These works typically make assumptions that are not applicable or require data that is not available in our setting. Several of these works analyze the graph of user reviews in order to detect fraudulent reviewers \cite{akoglu2015graph}. For example, \cite{hooi2016fraudar} detects fraud from the review graph by assuming that products are trying to maximize the number of positive reviews they get, whereas in our setting it is important to mitigate the effectiveness of just a single author-paper collusion since the number of reviews per paper is fixed and small. Additionally, some of these works \cite{kumar2018rev2, wang2011review, akoglu2013opinion} are based on estimating the ``true quality'' of each item from the review graph, which is not possible in peer review since paper evaluations are subjective. Another direction of prior work uses machine learning to detect malicious behavior. Some research \cite{yang2016re} detects fraud from reviewer statistics such as rating variance or number of ratings, but in the peer review setting, there is so little data for each reviewer that such features would be highly noisy or completely uninformative.  Other works \cite{ott2011finding} attempt to detect malicious reviews from the review text, but in our setting, malicious reviews may not differ stylistically from genuine reviews. Finally, some work in recommender systems focuses on making product recommendations resistant to malicious reviews \cite{si2020shilling}; however, in peer review, the paper acceptance process must be done by hand since it must take into account reviewer opinions, arguments, and the reviewer discussion.

Randomized assignments have been used to address the problem of fair division of indivisible goods such as jobs or courses \cite{hylland1979efficient, bogomolnaia2001new}, as well as in the context of Stackelberg security games \cite{korzhyk2010complexity}. The paper \cite{wang2018your} uses randomization to address the issue of miscalibration in ratings, such as those given to papers in peer review. To the best of our knowledge, the use of randomized reviewer-paper assignments to address the issues of malicious reviewers or reviewer de-anonymization in peer review has not been studied previously.
Work on randomized assignments often references the well-known Birkhoff-von Neumann theorem \cite{birkhoff1946three, von1953certain} or a generalization in order to demonstrate how to implement a randomized assignment as a lottery over deterministic assignments. The paper \cite{Budish2009IMPLEMENTINGRA} proposes a broad generalization of the Birkhoff-von Neumann theorem that we use in our work.

\section{Background and Problem Statements} \label{secproblem}
We first define the standard paper assignment problem, followed by our problem setting.
In the standard paper assignment setting, we are given a set $\revset$ of $\numrev$ reviewers and a set $\papset$ of $\numpap$ papers, along with desired reviewer load $\revload$ (that is, the maximum number of papers any reviewer should be assigned) and desired paper load $\papload$ (that is, the exact number of reviewers any paper should be assigned to).\footnote{For ease of exposition, we assume that all reviewer and paper loads are equal. In practice, program chairs may want to set different loads for different reviewers or papers; all of our algorithms and guarantees still hold for this case (as does our code).} An assignment of papers to reviewers is a bipartite matching between the sets that obeys the load constraints on all reviewers and papers. In addition, we are given a similarity matrix $\simmat \in \mathbb{R}^{\numrev \times \numpap}$ where $\simmat_{\adrev\adpap}$ denotes how good of a match reviewer $\adrev$ is for paper $\adpap$. These similarities can be derived from the reviewers' bids on papers, prior publications, conflicts of interest, etc.

The standard problem of finding a maximum sum-similarity assignment \cite{charlin2013toronto, charlin2012framework, goldsmith2007ai, flach2010novel, kobren19localfairness} is then written as an integer linear program. The decision variables $\intas \in \{0, 1\}^{\numrev \times \numpap}$ specify the assignment, where $\intas_{\adrev\adpap} = 1$ if and only if reviewer $\adrev$ is assigned to paper $\adpap$. The objective is to maximize $\sum_{\adrev \in \revset} \sum_{\adpap \in \papset} \simmat_{\adrev\adpap} \intas_{\adrev\adpap}$ subject to the load constraints $\sum_{\adpap \in \papset} \intas_{\adrev\adpap} \leq \revload, \forall \adrev \in \revset$ and $\sum_{\adrev \in \revset} \intas_{\adrev\adpap} = \papload, \forall \adpap \in \papset$.
Since the constraint matrix of the linear program (LP) relaxation of this problem is totally unimodular, the solution to the LP relaxation will be integral and so this problem can be solved as an LP. This method of assigning papers has been used by various conferences such as NeurIPS, ICML, ICCV, and SIGKDD (among others) \cite{charlin2013toronto, flach2010novel}, as well as by popular conference management systems EasyChair (\url{easychair.org}) and HotCRP (\url{hotcrp.com}).

Now, suppose there exists a reviewer who wishes to get assigned to a specific paper for some malicious reason and manipulates their similarities in order to do so. 
When the assignment algorithm is deterministic, as in previous work~\cite{charlin2013toronto, charlin2012framework, goldsmith2007ai, flach2010novel, kobren19localfairness, tang2010expertise}, a malicious reviewer who knows the algorithm may be able to effectively manipulate it in order to get assigned to the desired paper.
To address this issue, we aim to provide a guarantee that regardless of the reviewer bids and similarities, this reviewer-paper pair has only a limited probability of being assigned. 

Consider now the challenge of preserving anonymity in releasing conference data. If a conference releases its similarity matrix and its deterministic assignment algorithm, then anyone could reconstruct the full paper assignment. 
Interestingly, this problem can be solved in the same way as the malicious reviewer problems described above. If the assignment algorithm provides a guarantee that each reviewer-paper pair has only a limited probability of being assigned, then no reviewer's identity can be discovered with certainty.

With this motivation, we now consider $\intas$ as stochastic and aim to find a {\it randomized assignment}, a probability distribution over deterministic assignments. This naturally leads to the following problem formulation.
\begin{defn}[Pairwise-Constrained Problem] \label{defnpairwise}
The input to the problem is a similarity matrix $\simmat$ and a matrix $\maxprob \in [0, 1]^{\numrev \times \numpap}$. The goal is to find a randomized assignment of papers to reviewers (i.e., a distribution of $\intas$) that maximizes $\mathbb{E}\left[\sum_{\adrev \in \revset} \sum_{\adpap \in \papset} \simmat_{\adrev\adpap} \intas_{\adrev\adpap}\right]$ subject to the constraints $\mathbb{P}[\intas_{\adrev\adpap} = 1] \leq \maxprob_{\adrev\adpap}, \forall \adrev \in \revset, \adpap \in \papset$.
\end{defn}
Since a randomized assignment is a distribution over deterministic assignments, all assignments $\intas$ in the support of the randomized assignment must still obey the load constraints $\sum_{\adpap \in \papset} \intas_{\adrev\adpap} \leq \revload, \forall \adrev \in \revset$ and $\sum_{\adrev \in \revset} \intas_{\adrev\adpap} = \papload, \forall \adpap \in \papset$. The optimization objective is the expected sum-similarity across all paper-reviewer pairs, the natural analogue of the deterministic sum-similarity objective. In practice, the matrix $\maxprob$ is provided by the program chairs of the conference; all entries can be set to a constant value if the chairs have no special prior information about any particular reviewer-paper pair.

\paragraph{}
To prevent dishonest reviews of papers, program chairs may want to do more than just control the probability of individual paper-reviewer pairs. For example, suppose that we have three reviewers assigned per paper (a very common arrangement in computer science conferences). We might not be particularly concerned about preventing any single reviewer from being assigned to some paper, since even if that reviewer dishonestly reviews the paper, there are likely two other honest reviewers who can overrule the dishonest one. However, it would be much worse if we have two reviewers dishonestly reviewing the same paper, since they could likely overrule the sole honest reviewer.

A second issue is that there may be dependencies within certain pairs of reviewers that cannot be accurately represented by constraints on individual reviewer-paper pairs.  For example, we may have two reviewers $a$ and $b$ who are close collaborators, each of which we are not individually very concerned about assigning to paper $\adpap$. However, we may believe that in the case where reviewer $a$ has entered into a quid-pro-quo deal to dishonestly review paper $\adpap$, reviewer $b$ is likely to also be involved in the same deal. Therefore, one may want to strictly limit the probability that {\bf both} reviewers $a$ and $b$ are assigned to paper $\adpap$, regardless of the limits on the probability that either reviewer individually is assigned to paper $\adpap$.

With this motivation, we define the following generalization of the Pairwise-Constrained Problem.
\begin{defn}[Triplet-Constrained Problem] \label{defntriplet}
The input to the problem is a similarity matrix $\simmat$, a matrix $\maxprob \in [0, 1]^{\numrev \times \numpap}$, and a $3$-dimensional tensor $\maxprobtens \in [0, 1]^{\numrev \times \numrev \times \numpap}$. The goal is to find a randomized assignment of papers to reviewers that maximizes $\mathbb{E}\left[\sum_{\adrev \in \revset} \sum_{\adpap \in \papset} \simmat_{\adrev\adpap} \intas_{\adrev\adpap}\right]$ subject to the constraints $\mathbb{P}[\intas_{\adrev\adpap} = 1] \leq \maxprob_{\adrev\adpap}, \forall \adrev \in \revset, \adpap \in \papset$ and $\mathbb{P}[\intas_{a\adpap} = 1 \land \intas_{b\adpap} = 1 ] \leq \maxprobtens_{ab\adpap}, \forall a, b \in \revset \text{ s.t. } a \neq b, \adpap \in \papset$.
\end{defn}

\paragraph{}
The randomized assignments that solve these problems can be used to address all three challenges we identified earlier:
\begin{itemize}
    \item {\bf Untruthful favorable reviews:} By guaranteeing a limit on the probability that any malicious reviewer or any malicious pairs of reviewers can be assigned to the paper they want, we mitigate the effectiveness of any unethical deals between reviewers and authors by capping the probability that such a deal can be upheld. This guarantee holds regardless of how extreme a reviewers' manipulation of the assignment is and without any assumptions on reviewers' exact incentives. The entries of $\maxprob$ can be set by the program chairs based on their assessment of the risk of allowing the corresponding reviewer-paper pair; for example, an entry can be set low if the reviewer and author have collaborated in the past. The entries of $\maxprobtens$ can be set similarly based on known associations between reviewers.
    \item {\bf Torpedo reviewing:} By limiting the probability that any reviewer or pair of reviewers can be assigned to a paper they wish to torpedo, we make it much more difficult for a small group of reviewers to shut down a new research direction or to take out competing papers.
    \item {\bf Reviewer de-anonymization in releasing assignment data:} To allow for the release of similarities and the assignment algorithm after a conference, all of the entries in $\maxprob$ can simply be set to some reasonable constant value. Even if reviewer and paper names are fully identified through analysis of the similarities, only the distribution over assignments can be recovered and not the specific assignment that was actually used. This guarantees that for each paper, no reviewer's identity can be identified with high confidence, since every reviewer has only a limited chance to be assigned to that paper. 
\end{itemize}

In Sections~\ref{secbasicalgo} and \ref{secinstalgo}, we consider the Pairwise-Constrained Problem and Triplet-Constrained Problem respectively. We also consider several related problems in the appendices.
\begin{itemize}
    \item We extend our results to an objective based on \emph{fairness}, which we call the stochastic fairness objective, in Appendix~\ref{fairnessapdx}. Following the max-min fairness concept, we aim to maximize the minimum expected similarity assigned to any paper under the randomized assignment: $\min_{\adpap \in \papset} \mathbb{E}\left[ \sum_{\adrev \in \revset} \simmat_{\adrev\adpap} \intas_{\adrev\adpap}\right]$. We present a version of the Pairwise-Constrained Problem using this objective and an algorithm to solve it, as well as experimental results. 
    \item We address an alternate version of the Pairwise-Constrained Problem in Appendix~\ref{badmatchapdx} which uses the probabilities with which any reviewer may intend to untruthfully review any paper, along with other problems using these probabilities.
\end{itemize}

\section{Randomized Assignment with Reviewer-Paper Constraints} \label{secbasicalgo}

In this section we present our main algorithm to solve the Pairwise-Constrained Problem (Definition~\ref{defnpairwise}), thereby addressing the challenges identified earlier. Before delving into the details of the algorithm, the following theorem states our main result. 
\begin{thm}\label{bvnthm}
There exists an algorithm which returns an optimal solution to the Pairwise-Constrained Problem in $poly(\numrev, \numpap)$ time.
\end{thm}
We describe the algorithm, thereby proving this result, in the next two subsections. Our algorithm that realizes this result has two parts. In the first part, we find an optimal ``fractional assignment matrix,'' which gives the marginal probabilities of individual reviewer-paper assignments. The second part of the algorithm then samples an assignment, respecting the marginal probabilities specified by this fractional assignment.

\subsection{Finding the Fractional Assignment}

Define a {\it fractional assignment matrix} as a matrix $\fracas \in [0, 1]^{\numrev \times \numpap}$ that obeys the load constraints $\sum_{\adpap \in \papset} \fracas_{\adrev\adpap} \leq \revload$ for all reviewers $\adrev \in \revset$ and $\sum_{\adrev \in \revset} \fracas_{\adrev\adpap} = \papload$ for all papers $\adpap \in \papset$. Note that any deterministic assignment can be represented by a fractional assignment matrix with all entries in \{0, 1\}. Any randomized assignment is associated with a fractional assignment matrix where $\fracas_{\adrev\adpap}$ is the marginal probability that reviewer $\adrev$ is assigned to paper $\adpap$. Furthermore, randomized assignments associated with the same fractional assignment matrix have the same expected sum-similarity. The paper \cite{Budish2009IMPLEMENTINGRA} proves an extension of the Birkhoff-von Neumann theorem which shows that all fractional assignment matrices are implementable, i.e., they are associated with at least one randomized assignment. 
On the other hand, any probability matrix not obeying the load constraints cannot be implemented by a lottery over deterministic assignments, since all deterministic assignments do obey the constraints. 
Therefore, finding the optimal randomized assignment is equivalent to solving the following LP, which we call $\mathcal{LP}1$:
\begin{align}
\argmax_{\fracas \in \mathbb{R}^{\numrev \times \numpap}} \quad & \sum_{\adpap \in \papset} \sum_{\adrev \in \revset} \simmat_{\adrev\adpap} \fracas_{\adrev\adpap} \label{obj:LP1}\\
\text{subject to } \quad & 0 \leq \fracas_{\adrev\adpap} \leq 1 & \forall \adrev \in \revset, \forall \adpap \in \papset \label{cnt:LP1_range}\\
&\sum_{\adpap \in \papset} \fracas_{\adrev\adpap} \leq \revload & \forall \adrev \in \revset \label{cnt:LP1_colsum}\\
& \sum_{\adrev \in \revset} \fracas_{\adrev\adpap} = \papload & \forall \adpap \in \papset \label{cnt:LP1_rowsum}\\
& \fracas_{\adrev\adpap} \leq \maxprob_{\adrev\adpap} & \forall \adrev \in \revset, \forall \adpap \in \papset \label{cnt:LP1_Q}.
\end{align}

$\mathcal{LP}1$ has $O(\numpap\numrev)$ variables and $O(\numpap\numrev)$ constraints. Using techniques from \cite{jiang2020faster}, $\mathcal{LP}1$ can be solved in $O((\numpap\numrev)^{2.055})$ time.

\subsection{Implementing the Probabilities} \label{secbasicsamplingalgo}
$\mathcal{LP}1$ only finds the optimal marginal assignment probabilities $\fracas$ (where $\fracas$ now refers to a solution to $\mathcal{LP}1$). It remains to show whether and how these marginal probabilities can be implemented as a randomization over deterministic paper assignments. 
The paper \cite{Budish2009IMPLEMENTINGRA} provides a method for sampling a deterministic assignment from a fractional assignment matrix, which completes our algorithm once applied to the optimal solution of $\mathcal{LP}$1. Here we propose a simpler version of the sampling algorithm. Pseudocode for the algorithm is presented as Algorithm~\ref{algsampling}; we describe the algorithm in detail below. In Appendix~\ref{bvnthmproof}, we present a supplementary algorithm to compute the full distribution over deterministic assignments, which \cite{Budish2009IMPLEMENTINGRA} does not. Knowing the full distribution may be useful in order to compute other properties of the randomized assignment not calculable from $\fracas$ directly. 

\begin{algorithm}[t]
{\bf Input:} Fractional assignment matrix $\fracas$, reviewer set $\revset$, paper set $\papset$ \\
{\bf Ouput:} Deterministic assignment matrix $\intas$ \\
{\bf Algorithm:}
\begin{algorithmic}[1]
\State Construct vertex set $V \gets \revset \cup \papset \cup \{s\} \cup \{t\}$ \label{lineconstvertices} 
\State Construct directed edge set $E \gets \{(\adrev, \adpap) | \forall \adrev \in \revset, \adpap \in \papset\} \cup \{(s, \adrev) | \forall \adrev \in \revset\} \cup \{(\adpap, t) | \forall \adpap \in \papset\}$ \label{lineconstedges} 
\State Construct capacity function $\capac: E \to \mathbb{Z}$ as $\capac(e) \gets \begin{cases} 
1 & \text{if } e \in \revset \times \papset \\
\revload & \text{if } e \in \{s\} \times \revset \\
\papload  & \text{if } e \in \papset \times \{t\}
\end{cases}$ \label{lineconstcap} 
\State Construct initial flow function $f: E \to \mathbb{R}$ as $f(e) \gets \begin{cases}
\fracas_{\adrev\adpap} & \text{if } e = (\adrev, \adpap) \in \revset \times \papset \\
\sum_{\adpap \in \papset} \fracas_{\adrev\adpap} & \text{if } e = (s, \adrev) \in \{s\} \times \revset \\
\sum_{\adrev \in \revset} \fracas_{\adrev\adpap} & \text{if } e = (\adpap, t) \in \papset \times \{t\}
\end{cases}$ \label{lineconstflow} 
\While {$\exists e \in E$ such that $f(e) \not\in \mathbb{Z}$}
    \State Find a cycle of edges (ignoring direction) $C = \{e_1, \dots, e_k\}$ such that $f(e_i) \not\in \mathbb{Z}, \forall i \in [k]$ \label{linecyclefind}
    \State $A \gets \{e \in C | \text{ $e$ is directed in the same direction as $e_1$ along the cycle}\}$ 
    \State $B \gets C \setminus A$
    \State $\alpha \gets \min\left( \min_{e \in A}  f(e), \min_{e \in B} \capac(e) - f(e) \right)$ \label{lineflowamt1}
    \For {$e \in A$}
        \State $f_1(e) \gets f(e) - \alpha$
    \EndFor 
    \For {$e \in B$}
        \State $f_1(e) \gets f(e) + \alpha$
    \EndFor 
    \State $\beta \gets \min\left( \min_{e \in A}  \capac(e) - f(e), \min_{e \in B} f(e) \right)$ \label{lineflowamt2}
    \For {$e \in A$}
        \State $f_2(e) \gets f(e) + \beta$
    \EndFor 
    \For {$e \in B$}
        \State $f_2(e) \gets f(e) - \beta$
    \EndFor 
    \State $\gamma \gets \frac{\beta}{\alpha + \beta}$ \label{linegamma}
    \State With probability $\gamma$, $f \gets f_1$; else $f \gets f_2$ \label{linesetflow}
\EndWhile
\State $\intas_{\adrev\adpap} = f((\adrev, \adpap)), \forall (\adrev, \adpap) \in \revset \times \papset$ \label{lineconstructM}
\end{algorithmic}
\caption{Sampling algorithm for the Pairwise-Constrained Problem.}\label{algsampling}
\end{algorithm}

We begin by constructing a directed graph $G=(V, E)$ for our problem, along with a capacity function $\capac: E \to \mathbb{Z}$ (Lines~\ref{lineconstvertices}-\ref{lineconstcap}). First, construct one vertex for each reviewer, one vertex for each paper, and source and destination vertices $s, t$. Add an edge from the source vertex to each reviewer's vertex with capacity $\revload$. Add an edge from each paper's vertex to the destination vertex with capacity $\papload$. Finally, add an edge from each reviewer to each paper with capacity $1$. We also construct a flow function $f: E \to \mathbb{R}$, which obeys the flow conservation constraints $\sum_{e \in E \cap (V \times \{v\})} f(e) = \sum_{e \in E \cap (\{v\} \times V)} f(e), \forall v \in V \setminus \{s, t\}$ and the capacity constraints $f(e) \leq \capac(e), \forall e \in E$ (Line~\ref{lineconstflow}). A (possibly fractional) assignment $\fracas$ can be represented as a flow on this graph, where the flow from reviewer $i$ to paper $j$ corresponds to the probability reviewer $i$ is assigned to paper $j$ and the other flows are set uniquely by flow conservation. Due to the load constraints on assignments, the flows on the edges from the papers to the destination must be equal to those edges' capacities and the flows on the edges from the source to the reviewers must be less than or equal to the capacities. 

The algorithm then proceeds in an iterative manner, modifying the flow function $f$ on each iteration. On each iteration, we first check if there exists a ``fractional edge,'' an edge with non-integral flow. If no such edge exists, our current assignment is integral and so we can stop iterating. If there does exist a fractional edge, we then find an arbitrary cycle of fractional edges, ignoring direction (Line~\ref{linecyclefind}); this can be done by starting at any fractional edge and walking along fractional edges until a previously-visited vertex is returned to. On finding a cycle, we randomly modify the flow on each of the edges in the cycle in order to guarantee that at least one of the flows becomes integral. In what follows, we first prove that such a cycle of fractional edges can always be found. We then show how to modify the flows in order to guarantee the implementation of the marginal assignment probabilities.

We now show that a directionless cycle of fractional edges must exist whenever one fractional edge exists. Initially, by the properties of $\fracas$, the total flow on each edge going into vertex $t$ is integral; further, the algorithm only ever changes the flow on edges with non-integral flow. Therefore, the total flow going into $t$ is always integral. By flow conservation, the total flow leaving $s$ is also always integral. So, if there is a fractional edge adjacent to $s$, there must also be another fractional edge adjacent to $s$. As already stated, there are no fractional edges adjacent to $t$. Finally, for each vertex $v \in V \setminus \{s, t\}$, by flow conservation, there can never be only one fractional edge adjacent to $v$. Therefore, every vertex that is adjacent to a fractional edge must also be adjacent to another fractional edge. This proves that a directionless cycle of fractional edges must exist if one fractional edge exists.

We now show how to modify the flow on the edges in this cycle. We can keep pushing flow in some direction on this cycle (pushing negative flow if the edge is directed backwards) until some edge is at capacity or has $0$ flow. Call this amount of additional flow $\alpha$, and the resulting flow $f_1$. We can do the same thing in the other direction on the cycle, calling the additional flow $\beta$ and the resulting flow $f_2$. Both $f_1$ and $f_2$ must have at least one more integral edge than $f$, since some edge is at capacity. Further, both $f_1$ and $f_2$ obey the flow conservation and capacity constraints. Defining $\gamma \gets \frac{\beta}{\alpha + \beta}$, we set $f \gets f_1$ with probability $\gamma$ and $f \gets f_2$ with probability $1 - \gamma$ (Lines~\ref{linegamma}-\ref{linesetflow}).

Once all edges are integral (after the final iteration), we construct the sampled deterministic assignment $\intas$ from the flow on the reviewer-paper edges (Line~\ref{lineconstructM}). Since $f$ obeys the capacity constraints on all edges, $\intas$ obeys the load constraints and so is in fact an assignment. Since on each iteration the initial flow $f$ satisfies $f(e) = \gamma f_1(e) + \left( 1 - \gamma \right) f_2(e), \forall e \in E$, the expected final flow on each edge is always equal to the current flow on that edge. Since the expectation of a Bernoulli random variable is exactly the probability it equals one, each final reviewer-paper assignment $M_{\adrev\adpap}$ has been chosen with the desired marginal probabilities $\fracas_{\adrev\adpap}$.

Each iteration of this algorithm takes $O(\numpap + \numrev)$ time to find a cycle in the $O(\numpap + \numrev)$ vertices (if a list of fractional edges adjacent to each vertex is maintained), and it can take $O(\numpap \numrev)$ iterations to terminate since one edge becomes integral every iteration. Therefore, the sampling algorithm is overall $O(\numpap \numrev (\numpap + \numrev))$.

The time complexity of our full algorithm, including both $\mathcal{LP}1$ and the sampling algorithm, is dominated by the complexity of solving the LP. Since standard paper assignment algorithms such as TPMS can be implemented by solving an LP of the same size, our algorithm is comparable in complexity. If a conference currently does solve an LP to find their assignment, 
whatever LP solver a conference currently uses for their paper assignment algorithm could be used in our algorithm as well.

\section{Randomized Assignment with Constraints on Pairs of Reviewers} \label{secinstalgo}
We now turn to the problem of controlling the probabilities that certain pairs of reviewers are assigned to the same paper, defined in Section~\ref{secproblem} as the Triplet-Constrained Problem (Definition~\ref{defntriplet}). In the following subsections, we first show that the problem of finding an optimal randomized assignment given arbitrary constraints on the maximum probabilities of each reviewer-reviewer-paper grouping is NP-hard. We then show that, for the practical special case of restrictions on reviewers from the same subset of a partition of $\revset$ (such as the same primary academic institution or geographical area of residence), an optimal randomized assignment can be found efficiently.

\subsection{NP-Hardness of Arbitrary Constraints}

As described in Section~\ref{secproblem}, solving the Triplet-Constrained Problem would allow the program chairs of a conference maximum flexibility in how they control the probabilities of the assignments of pairs of reviewers. Unfortunately, as the following theorem shows, this problem cannot be efficiently solved. 



\begin{thm}\label{nph}
The Triplet-Constrained Problem is NP-hard, by reduction from 3-Dimensional Matching.
\end{thm}
3-Dimensional Matching is an NP-complete decision problem that takes as input three sets $X, Y, Z$ of size $s$ as well as a collection of tuples in $X \times Y \times Z$; the goal is to find a choice of $s$ tuples out of the collection such that no elements of any set are repeated \cite{karp1972reducibility}. Our reduction maps sets $X \cup Y$ to $\revset$ and $Z$ to $\papset$, and constructs $\maxprobtens \in \{0, 1\}^{\numrev \times \numrev \times \numpap}$ to allow only the assignments where the corresponding tuples are allowable in the 3-Dimensional Matching instance. The full proof is stated in Appendix~\ref{nphproof}.

Theorem~\ref{nph} implies a more fundamental result about the feasible region of implementable reviewer-reviewer-paper probability tensors, that is, the tensors $\fracastens \in [0, 1]^{\numrev \times \numrev \times \numpap}$ where entry $\fracastens_{ij\adpap}$ represents the marginal probability that both reviewers $i$ and $j$ are assigned to paper $\adpap$ under some randomized assignment. We can represent any deterministic assignment by a $3$-dimensional tensor $\intastens \in \{0, 1\}^{\numrev \times \numrev \times \numpap}$ where $\intastens_{ij\adpap} = 1$ if and only if both reviewers $i$ and $j$ are assigned to paper $\adpap$. Just as in the earlier case of fractional assignment matrices, the set of implementable probability tensors is a polytope with deterministic assignment tensors at the vertices (since any implementable probability tensor is a convex combination of deterministic assignment tensors). For fractional reviewer-paper assignment matrices, this polytope was defined by a small number ($O(\numpap \numrev)$) of linear inequalities, despite the fact that it has a large number of vertices (factorial in $\numpap$ and $\numrev$). However, this is no longer the case for reviewer-reviewer-paper probabilities. 
\begin{cor}\label{nphcor}
The polytope of implementable reviewer-reviewer-paper probabilities is not expressible in a polynomial (in $\numrev$ and $\numpap$) number of linear inequality constraints (assuming $P \neq NP$).
\end{cor}
The proof of this result is also stated in Appendix~\ref{nphproof}. 

\subsection{Constraints on Disjoint Reviewer Sets} \label{secpartitionproblem}
Since the most general problem of arbitrary constraints on reviewer-reviewer-paper triples is NP-hard, we must restrict ourselves to tractable special cases of interest. One such special case arises when the program chairs of a conference can partition the reviewers in such a way that they wish to prevent any two reviewers within the same subset from being assigned to the same paper. For example, reviewers can be partitioned by their primary academic institution. Since reviewers at the same institution are likely closely associated, program chairs may believe that placing them together as co-reviewers is more risky than would be implied by our concern about either reviewer individually. In this case, there may not even be any concern about the reviewers' motivations; the concern may simply be that the reviewers' opinions would not be sufficiently independent. Other partitions of interest could be the reviewer's geographical area of residence or research sub-field, as each of these defines a ``community'' of reviewers that may be more closely associated. This special case corresponds to instances of the Triplet-Constrained Problem where $\maxprobtens_{ab\adpap} = 0$ if reviewers $a$ and $b$ are in the same subset, and $\maxprobtens_{ab\adpap} = 1$ otherwise. 

We formally define this problem as follows:
\begin{defn}[Partition-Constrained Problem] \label{defnpartition}
The input to the problem is a similarity matrix $\simmat$, a matrix $\maxprob \in [0, 1]^{\numrev \times \numpap}$, and a partition of the reviewer set into subsets $\adinst_1, \dots, \adinst_\numinst \subseteq \revset$. The goal is to find a randomized assignment of papers to reviewers that maximizes $\mathbb{E}\left[\sum_{\adrev \in \revset} \sum_{\adpap \in \papset} \simmat_{\adrev\adpap} \intas_{\adrev\adpap}\right]$ subject to the constraints that $\mathbb{P}[\intas_{\adrev\adpap} = 1] \leq \maxprob_{\adrev\adpap}, \forall \adrev \in \revset, \adpap \in \papset$, and $\mathbb{P}[\intas_{a\adpap} = 1 \land \intas_{b\adpap} = 1] = 0, \forall a, b \in \adinst_i, \forall i \in [\numinst]$.
\end{defn}


For this special case of the Triplet-Constrained Problem, we show that the problem is efficiently solvable, as stated in the following theorem.
\begin{thm} \label{thminsts}
There exists an algorithm which returns an optimal solution to the Partition-Constrained Problem in poly(\numrev, \numpap) time. 
\end{thm}
We present the algorithm that realizes this result in the following subsections, thus proving the theorem. The algorithm has two parts: it first finds a fractional assignment matrix $\fracas$ meeting certain requirements, and then samples an assignment while respecting the marginal assignment probabilities given by $\fracas$ and additionally never assigning two reviewers from the same subset to the same paper. For ease of exposition, we first present the sampling algorithm, and then present an LP which finds the optimal fractional assignment matrix meeting the necessary requirements.

\subsubsection{Partition-Constrained Sampling Algorithm} \label{secinstsamplingalgo}
The sampling algorithm we present in this section takes as input a fractional assignment matrix $\fracas$ and samples an assignment while respecting the marginal assignment probabilities given by $\fracas$. The sampling algorithm is based on the following lemma:
\begin{lemma} \label{samplinglemma}
Consider any fractional assignment matrix $\fracas$ and any partition of $\revset$ into subsets $\adinst_1, \dots, \adinst_\numinst$. 
\begin{enumerate}[label=(\roman*)]
    \item There exists a sampling algorithm that implements the marginal assignment probabilities given by $\fracas$ and runs in $O(\numpap \numrev (\numpap + \numrev))$ time such that, for all papers $\adpap\in \papset$ and subsets $\adinst \in \{\adinst_1, \dots, \adinst_\numinst\}$ where $\sum_{\adrev \in \adinst} \fracas_{\adrev\adpap} \leq 1$, the algorithm never samples an assignment assigning two reviewers from subset $\adinst$ to paper $\adpap$.
    \item For any sampling algorithm that implements the marginal assignment probabilities given by $\fracas$, for all papers $\adpap\in \papset$ and subsets $\adinst \in \{\adinst_1, \dots, \adinst_\numinst\}$ where $\sum_{\adrev \in \adinst} \fracas_{\adrev\adpap} > 1$, the expected number of pairs of reviewers from subset $\adinst$ assigned to paper $\adpap$ is strictly positive.
\end{enumerate}
\end{lemma}

The sampling algorithm which realizes Lemma~\ref{samplinglemma} has an additional helpful property, which holds \emph{simultaneously} for all papers and subsets. We state the property in the following corollary and make use of it later:
\begin{cor} \label{samplingcor}
For any fractional assignment matrix $\fracas$, the sampling algorithm that realizes Lemma~\ref{samplinglemma} minimizes the expected number of pairs of reviewers from subset $\adinst$ assigned to paper $\adpap$ simultaneously for all papers $\adpap\in \papset$ and subsets $\adinst \in \{\adinst_1, \dots, \adinst_\numinst\}$ among all sampling algorithms implementing the marginal assignment probabilities given by $\fracas$.
\end{cor}

We present the sampling algorithm that realizes these results here, and prove the guarantees stated in Lemma~\ref{samplinglemma} and Corollary~\ref{samplingcor} in Appendix~\ref{samplingthmproof}. This algorithm is a modification of the sampling algorithm from Theorem~\ref{bvnthm} presented earlier as Algorithm~\ref{algsampling}. 

We first provide some high-level intuition about the modifications to Algorithm~\ref{algsampling}. 
For any fractional assignment matrix $\fracas$, for any subset $\adinst$ and paper $\adpap$, the expected number of reviewers from subset $\adinst$ assigned to paper $\adpap$ is $\sum_{\adrev \in \adinst} F_{\adrev \adpap}$. 
This is equal to the initial load from subset $\adinst$ on paper $\adpap$ in Algorithm~\ref{algsampling} (that is, the sum of the flow on all edges from reviewers in subset $\adinst$ to paper $\adpap$).
Note that at Algorithm~\ref{algsampling}'s conclusion, when all edges are integral, the load from subset $\adinst$ on paper $\adpap$ is equal to the number of reviewers from subset $\adinst$ assigned to paper $\adpap$.
Therefore, if the fractional assignment $\fracas$ is such that the initial expected number of reviewers from subset $\adinst$ assigned to paper $\adpap$ is no greater than $1$ (as stated in part (i) of Lemma~\ref{samplinglemma}), we want to keep the load from subset $\adinst$ on paper $\adpap$ close to its initial value so that the final number of reviewers from subset $\adinst$ assigned to paper $\adpap$ is also no greater than $1$.
With this reasoning, we modify Algorithm~\ref{algsampling} so that in each iteration,  it ensures that the total load on each paper from each subset is unchanged if originally integral and is never moved past the closest integer in either direction if originally fractional. 


\begin{algorithm}[t]
\begin{algorithmic}[1]
\State Construct the set of undirected edges $E_U \gets E \cup \{(v, u) \mid (u, v) \in E\}$
\State Construct the undirected flow function $f_U: E_U \to \mathbb{R}$ as $f_U((u, v)) \gets \begin{cases} f((u, v)) & \text{if } (u, v) \in E \\ f((v, u)) & \text{otherwise} \end{cases}$
\State Find arbitrary edge $(u, v) \in E$ such that $f((u, v)) \not\in \mathbb{Z}$
\State $C \gets \{ (u, v) \}$
\State $D_1 \gets \{\}$, $D_2 \gets \{\}$
\While {$v$ has not previously been visited}
    \State Visit $v$
    \If {$u \in \revset$ and $v \in \papset$} \label{linecase1}
        \State Set $\adinst \in \{\adinst_1, \dots, \adinst_\numinst\}$ such that $u \in \adinst$
        \If {$\exists w \in \adinst \setminus \{ u \}$ such that $(v, w) \in E_U$ and $f_U((v, w)) \not\in \mathbb{Z}$}
            \State Find such a $w$
        \Else \label{linecase2}
            \State For some $J \in \{\adinst_1, \dots, \adinst_\numinst\} \setminus \{\adinst\}$ such that $\sum_{\adrev \in J} f((\adrev, v)) \not\in \mathbb{Z}$, find $w \in J$ such that $(v, w) \in E_U$ and $f_U((v, w)) \not\in \mathbb{Z}$
            \State $D_1 \gets D_1 \cup \{ \sum_{\adrev \in \adinst} f((\adrev, v)) \}$ (corresponding to $(u, v)$) \label{lined1}
            \State $D_2 \gets D_2 \cup \{ \sum_{\adrev \in J} f((\adrev, v)) \}$ (corresponding to $(v, w)$) \label{lined2}
        \EndIf
    \Else 
        \State Find $w \in V \setminus \{ u \}$ such that $(v, w) \in E_U$ and $f_U((v, w)) \not\in \mathbb{Z}$
    \EndIf 
    \State $C \gets C \cup \{ (v, w) \}$
    \State $u \gets v$
    \State $v \gets w$
\EndWhile
\State Set $e_1$ as the first edge in $C$ leaving $v$
\State Set $e_{-1}$ as the last edge in $C$ (entering $v$)
\State Remove edges preceding $e_1$ from $C$, and remove the corresponding elements from $D_1$ and $D_2$
\If {$v \in \papset$ and $\exists \adinst \in \{\adinst_1, \dots, \adinst_\numinst\}$ such that $e_1 \in \{v \} \times \adinst$ and $e_{-1} \in \adinst \times \{v\}$}
    \State Remove the elements corresponding to $e_1$ and $e_{-1}$ from $D_1$ and $D_2$
\EndIf 
\If {$e_1 \not\in E$}
    \State Swap $D_1$ and $D_2$ 
\EndIf
\State Replace each edge in $C$ from $E_U$  with the corresponding edge from $E$
\end{algorithmic}
\caption{Loop-finding subroutine (replacing Line~\ref{linecyclefind} in Algorithm~\ref{algsampling}).} \label{algsamplingmod}
\end{algorithm}

The algorithm realizing Lemma~\ref{samplinglemma} and Corollary~\ref{samplingcor} is obtained by changing three lines in Algorithm~\ref{algsampling}, as follows:
\begin{itemize}
    \item Line~\ref{linecyclefind} is replaced with the subroutine in Algorithm~\ref{algsamplingmod}.
    \item Line~\ref{lineflowamt1} is changed to: $\alpha \gets \min\left( \min_{e \in A}  f(e), \min_{e \in B} \capac(e) - f(e), \min_{t \in D_1} t - \lfloor t \rfloor, \min_{t \in D_2} \lceil t \rceil - t \right)$.
    \item Line~\ref{lineflowamt2} is changed to: $\beta \gets \min\left( \min_{e \in A}  \capac(e) - f(e), \min_{e \in B} f(e),  \min_{t \in D_1} \lceil t \rceil - t  , \min_{t \in D_2} t - \lfloor t \rfloor \right)$.
\end{itemize}

The primary modification we make to Algorithm~\ref{algsampling} is replacing Line~\ref{linecyclefind} with the subroutine in Algorithm~\ref{algsamplingmod}. In each iteration, when we look for an undirected cycle of fractional edges in the graph, we now choose the cycle carefully rather than arbitrarily. We find a cycle by starting from an arbitrary fractional edge in the graph and walk along adjacent fractional edges (ignoring direction) until we repeat a previously-visited vertex. As we do this, whenever we take a fractional edge from a reviewer in subset $\adinst$ into paper $\adpap$, there are two cases. 
\begin{itemize}
\item Case 1: If there exists a different fractional edge from paper $\adpap$ to subset $\adinst$ (Line~\ref{linecase1} in Algorithm~\ref{algsamplingmod}), we take this edge next. Note that if the total load from subset $\adinst$ on paper $\adpap$ is integral, such an edge must exist.
\item Case 2: Otherwise (Line~\ref{linecase2} in Algorithm~\ref{algsamplingmod}), we must take a fractional edge from paper $\adpap$ to some other subset $J$. In this case, the total load from subset $\adinst$ on paper $\adpap$ must not be integral. We choose the subset $J$ so that the total load from subset $J$ on paper $\adpap$ is also not integral. Such a subset must exist since the total load on paper $\adpap$ is always integral. We keep track of both the total load from $\adinst$ and from $J$ on $\adpap$, for every occurrence of this case along the cycle (Lines~\ref{lined1} and \ref{lined2} in Algorithm~\ref{algsamplingmod}). 
\end{itemize}
In Case 1, no matter how much flow is pushed on the cycle, the total load from subset $\adinst$ on paper $\adpap$ will be preserved exactly. However, due to Case 2, we must modify the choice of how much flow to push on the cycle to ensure that the loads are preserved as desired. Specifically, we only push flow in a given direction on the cycle until the total load for either subset $\adinst$ or $J$ on paper $\adpap$ is integral, for any $\adinst, J, \adpap$ found in Case 2. The total loads from each subset on each paper found in Case 2 are saved in either set $D_1$ or set $D_2$ depending on the direction of the corresponding edges in the cycle, and each subset-paper pair with an edge corresponding to an element of $D_1$ or $D_2$ has only that one edge in the cycle. If the total (fractional) load from subset $\adinst$ on paper $\adpap$ is $t$, then only $\lceil t \rceil - t$ additional flow can be added to any edge from subset $\adinst$ to paper $\adpap$ before the load becomes integral; similarly, only $t - \lfloor t \rfloor$ flow can be removed from any edge before the load becomes integral. This leads to the stated changes to Lines~\ref{lineflowamt1} and \ref{lineflowamt2} in Algorithm~\ref{algsampling}.

Therefore, on each iteration, we push flow until either the flow on some edge is integral (as in the original algorithm), or until the total load on some paper from some subset is integral. This implies that the algorithm still terminates in a finite number of iterations. In addition, by the end of the algorithm, the total load on each paper from each subset is preserved exactly if originally integral and rounded in either direction if originally fractional, as desired. 

The time complexity of this modified algorithm is identical to that of the original algorithm from Theorem~\ref{bvnthm}, since finding a cycle takes the same amount of time (if a fractional adjacency list for each subset is used) and only a maximum of $O(\numrev)$ extra iterations are performed (if an subset's total load becomes integral rather than an edge's flow). Therefore, the algorithm is overall $O(\numpap \numrev (\numpap + \numrev))$.

\subsubsection{Finding the Optimal Partition-Constrained Fractional Assignment} \label{secinstlpalgo}
Lemma~\ref{samplinglemma} provides necessary and sufficient conditions for the fractional assignment matrices for which it is possible to prevent all pairs of same-subset reviewers from being assigned to the same paper. 
Therefore, to find an optimal fractional assignment with this property, we just need to add $\numinst\numpap$ constraints to $\mathcal{LP}1$. We call this new LP $\mathcal{LP}2$:
\begin{align}
\argmax_{\fracas \in \mathbb{R}^{\numrev \times \numpap}} \quad & \sum_{\adpap \in \papset} \sum_{\adrev \in \revset} \simmat_{\adrev\adpap} \fracas_{\adrev\adpap}\\
\text{subject to } &\text{Constraints  (\ref{cnt:LP1_range}--\ref{cnt:LP1_Q}) from $\mathcal{LP}1$ and}\nonumber \\
& \sum_{\adrev \in \adinst} \fracas_{\adrev\adpap} \leq 1 \qquad \qquad \forall \adinst \in \{\adinst_1, \dots, \adinst_\numinst\}, \forall \adpap \in \papset.\label{cnt:inst}
\end{align}

The solution to $\mathcal{LP}2$ when paired with the sampling algorithm from Section~\ref{secinstsamplingalgo} never assigns two reviewers from the same subset to the same paper. Furthermore, since any fractional assignment $\fracas$ not obeying Constraint~(\ref{cnt:inst}) will have a strictly positive probability of assigning two reviewers from the same subset to the same paper, $\mathcal{LP}2$ finds the optimal fractional assignment with this guarantee. This completes the algorithm for the Partition-Constrained Problem.

Additionally, Corollary~\ref{samplingcor} shows that the sampling algorithm from Section~\ref{secinstsamplingalgo} is optimal in the expected number of same-subset reviewer pairs, for any fractional assignment. 
If the guarantee of entirely preventing same-subset reviewer pairs is not strictly required, Constraint~(\ref{cnt:inst}) in $\mathcal{LP}2$ can be loosened (constraining the subset loads to a higher value) without removing it entirely. For the resulting fractional assignment $\fracas$, the sampling algorithm from Section~\ref{secinstsamplingalgo} still minimizes the expected number of pairs of reviewers from any subset on any paper, as compared to any other sampling algorithm implementing $\fracas$. Since the subset loads are still constrained, the expected number of same-subset reviewer pairs will be lower than in the solution to the Pairwise-Constrained Problem (at the cost of some expected sum-similarity). We examine this tradeoff experimentally in Section~\ref{secexps}.

\section{Experiments} \label{secexps}

We test our algorithms on several real-world datasets. 
The first real-world dataset is a similarity matrix recreated from ICLR 2018 data in \cite{xu2018strategyproof}; this dataset has $\numrev = 2435$ reviewers and $\numpap = 911$ papers. We also run experiments on  similarity matrices created from reviewer bid data for three AI conferences from PrefLib dataset MD-00002 \cite{MaWa13a}, with sizes $(\numrev = 31, \numpap = 54)$, $(\numrev = 24, \numpap = 52)$, and $(\numrev = 146, \numpap = 176)$ respectively. For all three PrefLib datasets, we transformed ``yes,'' ``maybe,'' and ``no response'' bids into similarities of $4$, $2$, and $1$ respectively, as is often done in practice~\cite{shah2018design}. As done in \cite{xu2018strategyproof}, we set loads $\revload = 6$ and $\papload = 3$ for all datasets since these are common loads for computer science conferences (except on the PrefLib2 dataset, for which we set $\revload=7$ for feasibility). 

We run all experiments on a computer with $8$ cores and $16$ GB of RAM, running Ubuntu 18.04 and using Gurobi 9.0.2~\cite{gurobi} to solve the LPs. Our algorithm for the Pairwise-Constrained Problem takes an average of $41$ seconds to complete on ICLR; our algorithm for the Partition-Constrained Problem takes an average of $45$ seconds. As expected, the running time is dominated by the time taken to solve the LP.

\begin{figure*}[t!] 
    \centering
    \begin{subfigure}{0.65\textwidth}\includegraphics[width=1\textwidth]{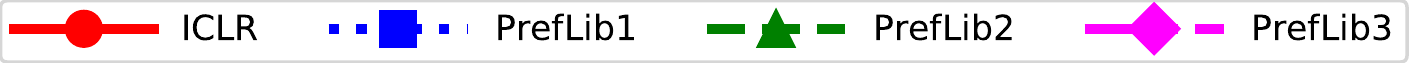}\label{figlegend} \end{subfigure} \\
    \begin{subfigure}{0.45\textwidth}\includegraphics[width=1\textwidth]{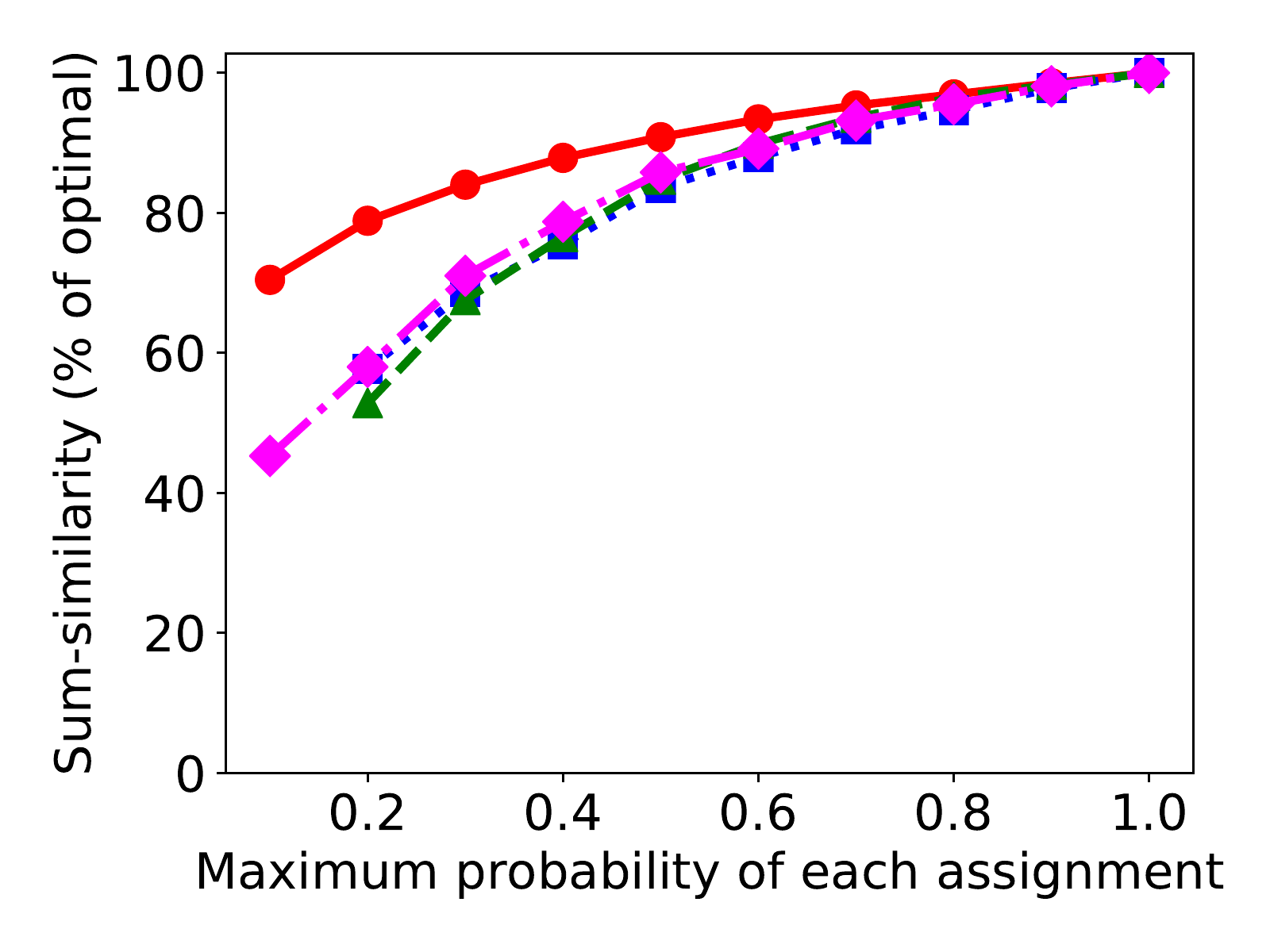}\caption{Pairwise-Constrained Problem}\label{figsim} \end{subfigure}\quad
    \begin{subfigure}{0.45\textwidth}\includegraphics[width=1\textwidth]{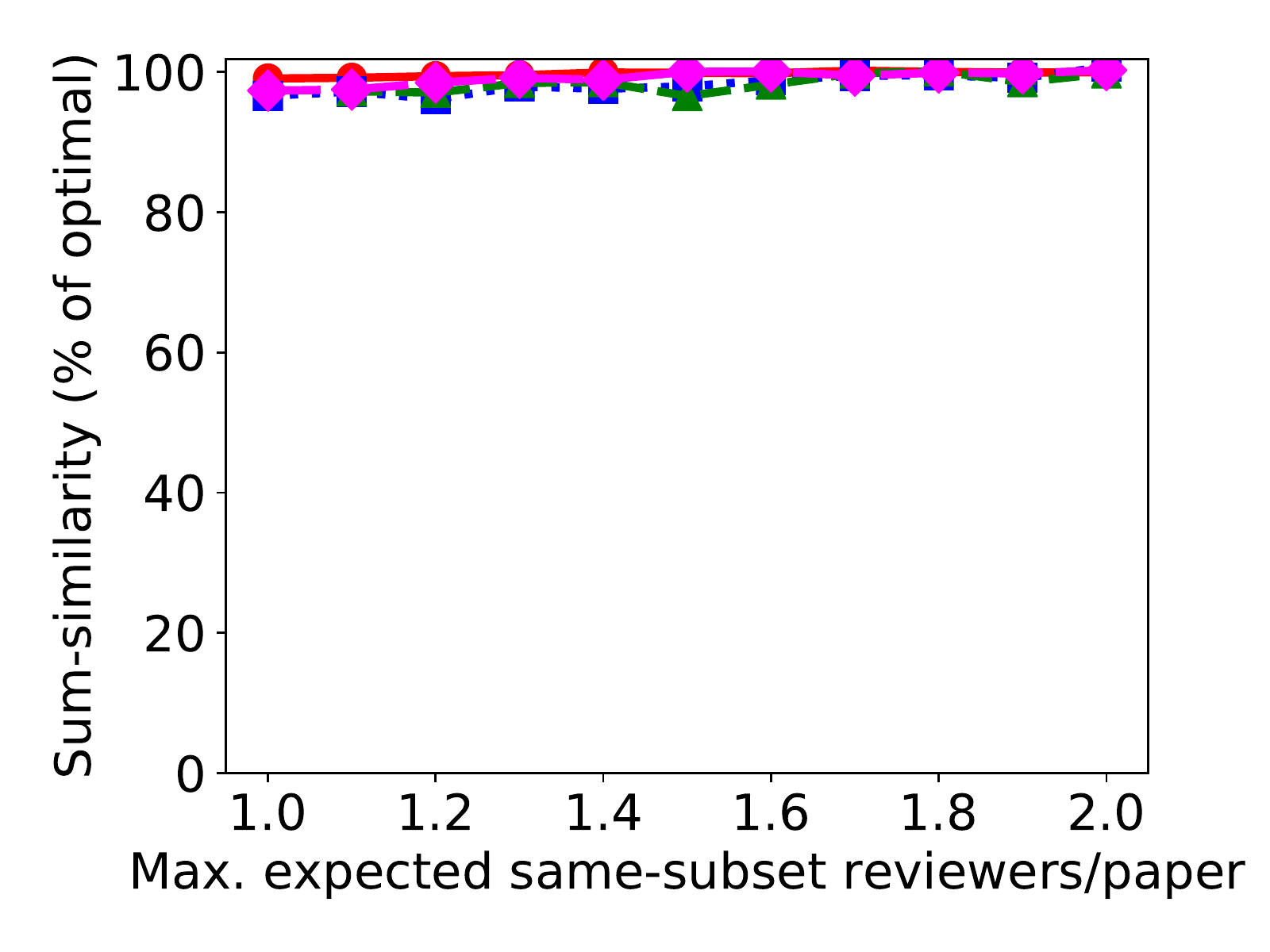}\caption{Partition-Constrained Problem with three random subsets}\label{figinst} \end{subfigure}
    \caption{Experimental results on four conference datasets.} \label{figresults}
\end{figure*}

\subsection{Quality of Resulting Assignments}
We first study our algorithm for the Pairwise-Constrained Problem, as described in Section~\ref{secbasicalgo}. In this setting, program chairs must make a tradeoff between the quality of the output assignments and guarding against malicious reviewers or reviewer de-anonymization by setting the values of the maximum-probability matrix $\maxprob$. We  investigate this tradeoff on real datasets. All results in this section are averaged over $10$ trials with error bars plotted representing the standard error of the mean, although they are sometimes not visible since the variance is very low.

In Figure~\ref{figsim}, we set all entries of the maximum-probability-matrix $\maxprob$ equal to the same constant value $\maxprobconst$ (varied on the x-axis), and observe how the sum-similarity value of the assignment computed via our algorithm from Section~\ref{secbasicalgo} changes as $\maxprobconst$ increases from $0.1$ to $1$ with an interval of $0.1$. We report the sum-similarity as a percentage of the unconstrained optimal solution's objective. This unconstrained optimal solution maximizes sum-similarity through a deterministic assignment as is popularly done today \cite{charlin2012framework, long2013good, goldsmith2007ai, tang2010expertise, flach2010novel,lian2018conference,kobren19localfairness}, and does not address the aforementioned challenges. We see that our algorithm trades off the maximum probability of an assignment gracefully against the sum-similarity on all datasets. For instance, with $\maxprobconst = 0.5$, our algorithm achieves $90.8\%$ of the optimal objective value on the ICLR dataset. In practice, this would allow the program chairs of a conference to limit the chance that any malicious reviewer is assigned to their desired paper to $50\%$ without suffering a significant loss of assignment quality. When $\maxprobconst$ is too small, a feasible assignment may not exist in some datasets (e.g., $\maxprobconst=0.1$ for PrefLib2). 

We next test our algorithm for the Partition-Constrained Problem discussed in Section~\ref{secpartitionproblem}. In this algorithm, program chairs can navigate an additional tradeoff between the number of same-subset reviewers assigned to the same paper and the assignment quality; we investigate this tradeoff here. On ICLR, we fix $\maxprobconst = 0.5$ and randomly assign reviewers to subsets of size $15$, using this as our partition of $\revset$ (since the dataset does not include any reviewer information). Each subset represents a group of reviewers with close associations, such as reviewers from the same institution. Our algorithm is able to achieve $100\%$ of the optimal objective for the Pairwise-Constrained Problem with $\maxprobconst = 0.5$ while preventing any pairs of reviewers from the same subset from being assigned to the same paper. 

Since our algorithm achieves the full possible objective in this setting, we now run experiments with a considerably more restrictive partition constraint. In Figure~\ref{figinst}, we show an extreme case where we randomly assign reviewers to $3$ subsets of equal size (sizes $811$, $11$, $8$ and $48$ on ICLR and the PrefLib datasets, respectively, with the remainder assigned to a dummy fourth subset), again fixing $\maxprobconst = 0.5$. We then gradually loosen the constraints on the expected number of same-subset reviewers assigned to the same paper by increasing the constant in Constraint~(\ref{cnt:inst}) from $1$ to $2$ in increments of $0.1$, shown on the x-axis. We plot the sum-similarity objective of the resulting assignment, expressed as a percentage of the optimal non-partition-constrained solution's objective (i.e., the solution to the Pairwise-Constrained Problem with $\maxprobconst = 0.5$). 
Even in this extremely constrained case with only a few subsets, we still achieve $99.1\%$ of the non-partition-constrained objective while entirely preventing same-subset reviewer pairs on ICLR.

In Appendix~\ref{addexperiments}, we present results for additional experiments on synthetic similarities, where we find results qualitatively similar to those presented here. We also run experiments for a fairness objective, which we present in Appendix~\ref{fairnessapdx}.

\subsection{Effectiveness at Preventing Manipulation}

\begin{figure*}[t!] 
    \centering
    \begin{subfigure}{0.5\textwidth}\includegraphics[width=1\textwidth]{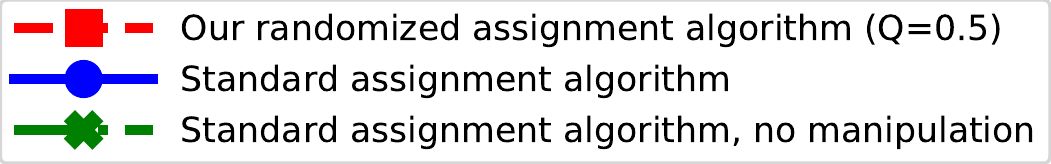}\label{figmaniplegend} \end{subfigure} \\
    \begin{subfigure}{0.45\textwidth}\includegraphics[width=1\textwidth]{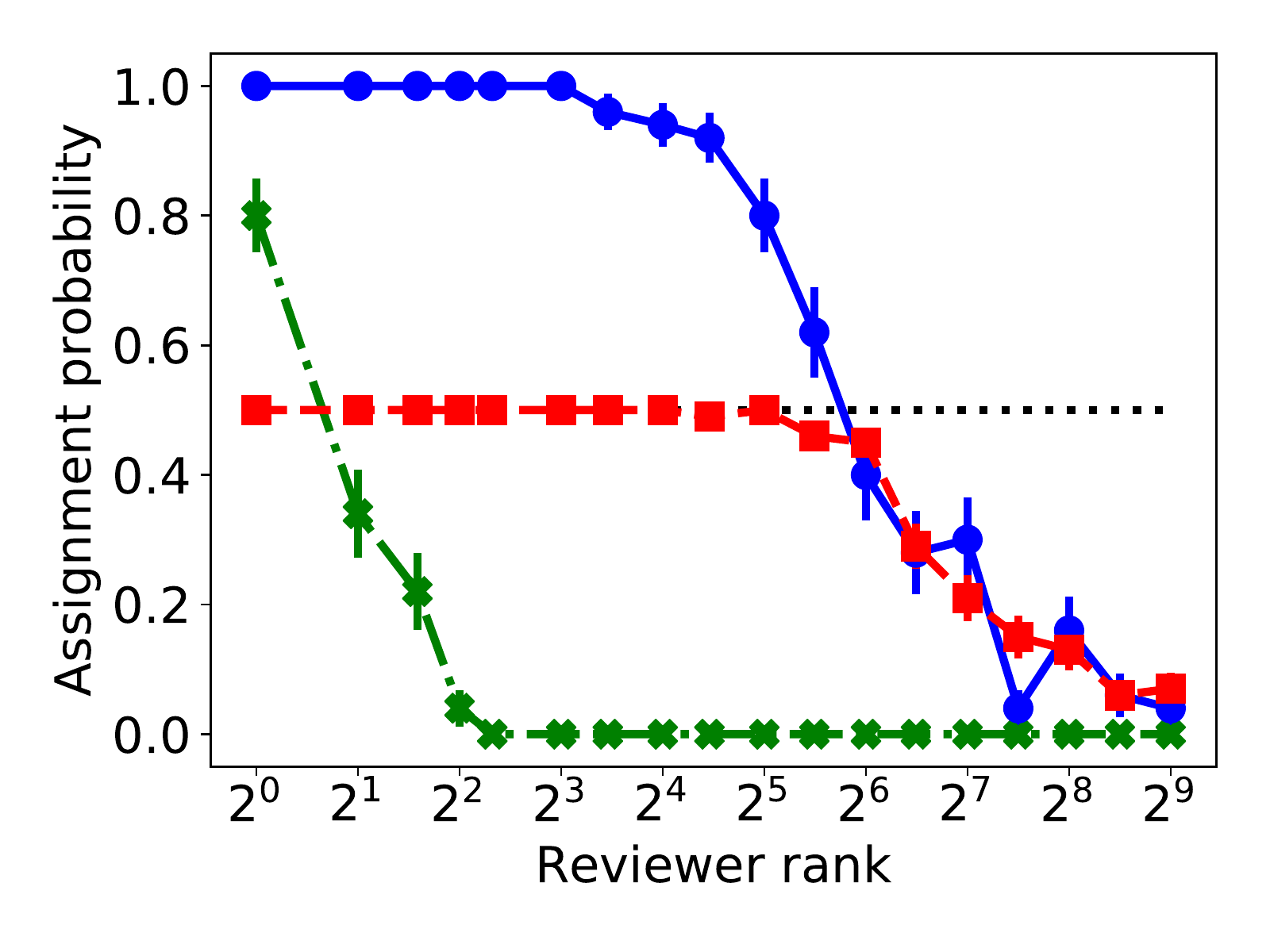}\caption{Bidding scale $\bidscale=2$}\label{figmanip2} \end{subfigure}\quad
    \begin{subfigure}{0.45\textwidth}\includegraphics[width=1\textwidth]{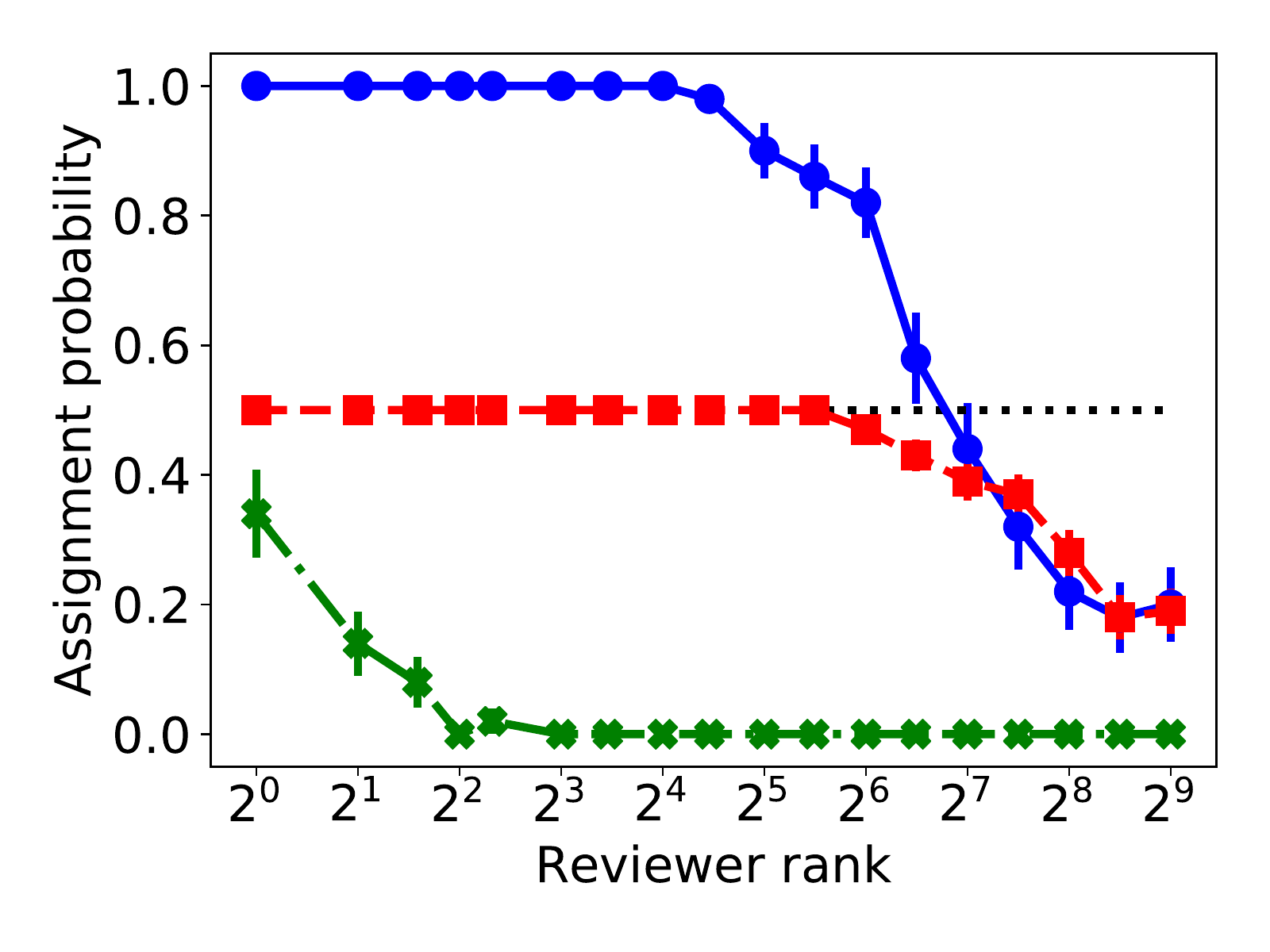}\caption{Bidding scale $\bidscale=4$}\label{figmanip4} \end{subfigure}
    \caption{Effectiveness of bidding manipulation on the ICLR dataset. One malicious reviewer manipulates their bids to get assigned to a target paper. The probability of the malicious reviewer-target paper assignment varies on the y-axis as the rank of the malicious reviewer's pre-bid similarity with the target paper changes on the x-axis.} \label{figmanip}
\end{figure*}

We now describe experiments evaluating the effectiveness of our algorithm at preventing manipulation on the ICLR dataset against a simulated reviewer bidding model. We assume that there is one malicious reviewer who is attempting to maximize their chances of being assigned to a target paper solely through bidding (and not through other means). Since the ICLR similarities are reconstructed purely from the text similarity with each reviewers' past work and do not contain any bidding, we supplement them with synthetic bids. Specifically, each reviewer $\adrev$ chooses a bid $b_{\adrev \adpap} \in \{-1, 0, 1\}$ for each paper $\adpap$, indicating ``not interested,'' ``neutral,'' or ``interested'' respectively. Based on the similarity function used in the NeurIPS 2016 conference~\cite{shah2018design}, we compute the final similarity between reviewer $\adrev$ and paper $\adpap$ as $S'_{\adrev \adpap} = \bidscale^{b_{\adrev \adpap}} S_{\adrev \adpap}$, where $S_{\adrev \adpap}$ is the text similarity from the ICLR dataset and $\bidscale$ is a fixed scale parameter.

In our experiment, the malicious reviewer bids $1$ on their target paper and $-1$ on all other papers. The other (honest) reviewers bid according to a simple randomized model constructed to match characteristics of the bidding observed in NeurIPS 2016~\cite{shah2018design}. We divide the reviewers uniformly at random into three groups. The first group contains $20\%$ of the reviewers, who all bid $0$ on all papers. The second group contains $50\%$ of the reviewers, who bid non-zero on a low number of papers. These reviewers consider each paper within the $10\%$ of papers that have highest text similarity with them, and independently choose to bid non-zero on each one with probability $0.016$. If a paper is selected to bid non-zero, the bid is chosen from $\{-1, 1\}$ with uniform probability. The third group contains $30\%$ of the reviewers, who bid non-zero on a high number of papers. They follow the same bidding procedure as the second group, but bid with probability $0.24$. 

The results of this experiment are shown in Figure~\ref{figmanip}. We choose a target paper uniformly at random, and choose the malicious reviewer to be the reviewer with the $x$\textsuperscript{th} highest text similarity with that paper (varying $x$ on the x-axis). Note that the x-axis is on a log-scale. We then have all reviewers bid in the manner described above, and compute the assignment with either the standard deterministic assignment algorithm described in Section~\ref{secproblem} or our randomized assignment algorithm for the Pairwise-Constrained Problem, setting all entries of $\maxprob$ to $0.5$. We then observe the probability that the malicious reviewer is assigned to the target paper (that is, the probability with which the manipulation is successful), which is in $\{0, 1\}$ for the deterministic algorithm but can be non-integral for our algorithm. For each point on the x-axis, we average results over $50$ choices of target paper, giving an overall success rate for the manipulation under a uniform choice of papers (reported on the y-axis, with error bars plotted representing the standard error of the mean). For comparison, we also plot the case where only the honest reviewers bid and the malicious reviewer does not bid. 

There are three key takeaways from this experiment. First, 
we see that {when a reviewer does not bid, their assignment probability is low} for any reviewer not ranked in the top 4 for that paper in terms of the text similarity. Second, when the malicious reviewer does bid, {the manipulation has a high success rate  under the standard assignment algorithm}. For example, the $8$\textsuperscript{th} ranked reviewer for any paper is never assigned if they do not bid, but with bids they can manipulate in order to guarantee their assignment. Moreover, even the $100$\textsuperscript{th} ranked reviewer has a has a decent probability (above $0.25$) of getting assigned the target paper if the reviewer bids maliciously. This indicates that manipulation from reviewers is quite powerful in standard assignment algorithms, potentially compromising the integrity of the assignment.  Third, {our algorithm always limits the probability of successful manipulation} to the desired level of $0.5$, reflecting the theoretical guarantees presented earlier in the paper. For malicious reviewers who have low text similarity with the target paper (e.g., reviewers from a different subject area), our algorithm occasionally gives the manipulation a marginally higher probability to succeed as compared to the standard assignment algorithm since the set of possible reviewers for each paper is larger. However, manipulation from these low-similarity reviewers is unlikely to succeed in the first place (with probability below the desired limit of $0.5$), and it is envisaged to be easier for program chairs to manually spot unusual bids from reviewers outside of a paper's subject area.

\section{Discussion}
We have presented here a framework and a set of algorithms for addressing three challenges of practical importance to the peer review process: untruthful favorable reviews, torpedo reviewing, and reviewer de-anonymization on the release of assignment data. By design, our algorithms are quite flexible to the needs of the program chairs, depending on which challenges they are most concerned with addressing. Our empirical evaluations demonstrate some of the tradeoffs that can be made between total similarity and maximum probability of each paper-reviewer pair or between total similarity and  number of reviewers from the same subset on the same paper. The exact parameters of the algorithm can be set based on how the program chairs weigh the relative importance of each of these factors. Note that an empirical evaluation of exactly how much our algorithm reduces manipulation in a real conference is not possible, since the ground truth of which reviewers were manipulating their assignments is not known.

This work leads to a number of open problems of interest. First, since the general Triplet-Constrained Problem is NP-hard, we considered one special structure---the Partition-Constrained Problem---of practical relevance. A direction for future research is to find additional special cases under which optimizing over constraints on the probabilities of reviewer-pair-to-paper assignments is feasible. For example, there may be a known network of reviewers where program chairs wish to prevent connected reviewers from being assigned to the same paper.  
A second problem of interest is to develop methods to detect potentially malicious reviewer-paper pairs before papers are assigned (e.g., based on the bids). 
Finally, this work does not address the problem of reviewers colluding with each other to give dishonest favorable reviews after being assigned to each others' papers; we leave this issue for future work.


\section*{Acknowledgments}
The research of Steven Jecmen and Nihar Shah was supported in part by NSF CAREER 1942124. The research of Steven Jecmen and Fei Fang was supported in part by NSF Award IIS-1850477. The research of Hanrui Zhang and Vincent Conitzer was supported in part by NSF Award IIS-1814056.

\bibliographystyle{unsrt}
\bibliography{bibtex}

~\\~\\
\appendix
\noindent{\LARGE \bf Appendices}

\section{Stochastic Fairness Objective} \label{fairnessapdx}
An alternate objective to the sum-similarity objective has been studied in past work \cite{stelmakh2018peerreview4all, Garg2010papers}, aiming to improve the fairness of the assignment with respect to the papers. Rather than maximizing the sum-similarity across all papers, this objective maximizes the minimum total similarity assigned to any paper:
\begin{align*}
\argmax_{\intas \in \mathbb{R}^{\numrev \times \numpap}} \quad & \min_{\adpap \in \papset} \sum_{\adrev \in \revset} \simmat_{\adrev\adpap} \intas_{\adrev\adpap} \\
\text{subject to } \quad & \intas_{\adrev\adpap} \in \{0, 1\} & \forall \adrev \in \revset, \forall \adpap \in \papset &\\
&\sum_{\adpap \in \papset} \intas_{\adrev\adpap} \leq \revload & \forall \adrev \in \revset &\\
& \sum_{\adrev \in \revset} \intas_{\adrev\adpap} = \papload & \forall \adpap \in \papset &.
\end{align*}
Due to the minimum in the objective, this problem is NP-hard \cite{garg2010assigning}; the paper \cite{stelmakh2018peerreview4all} presents an algorithm to find an approximate solution.

\begin{figure*}[t!] 
    \centering
    \begin{subfigure}{0.65\textwidth}\includegraphics[width=1\textwidth]{images/legend.pdf}\end{subfigure} \\
    \begin{subfigure}{0.45\textwidth}\includegraphics[width=1\textwidth]{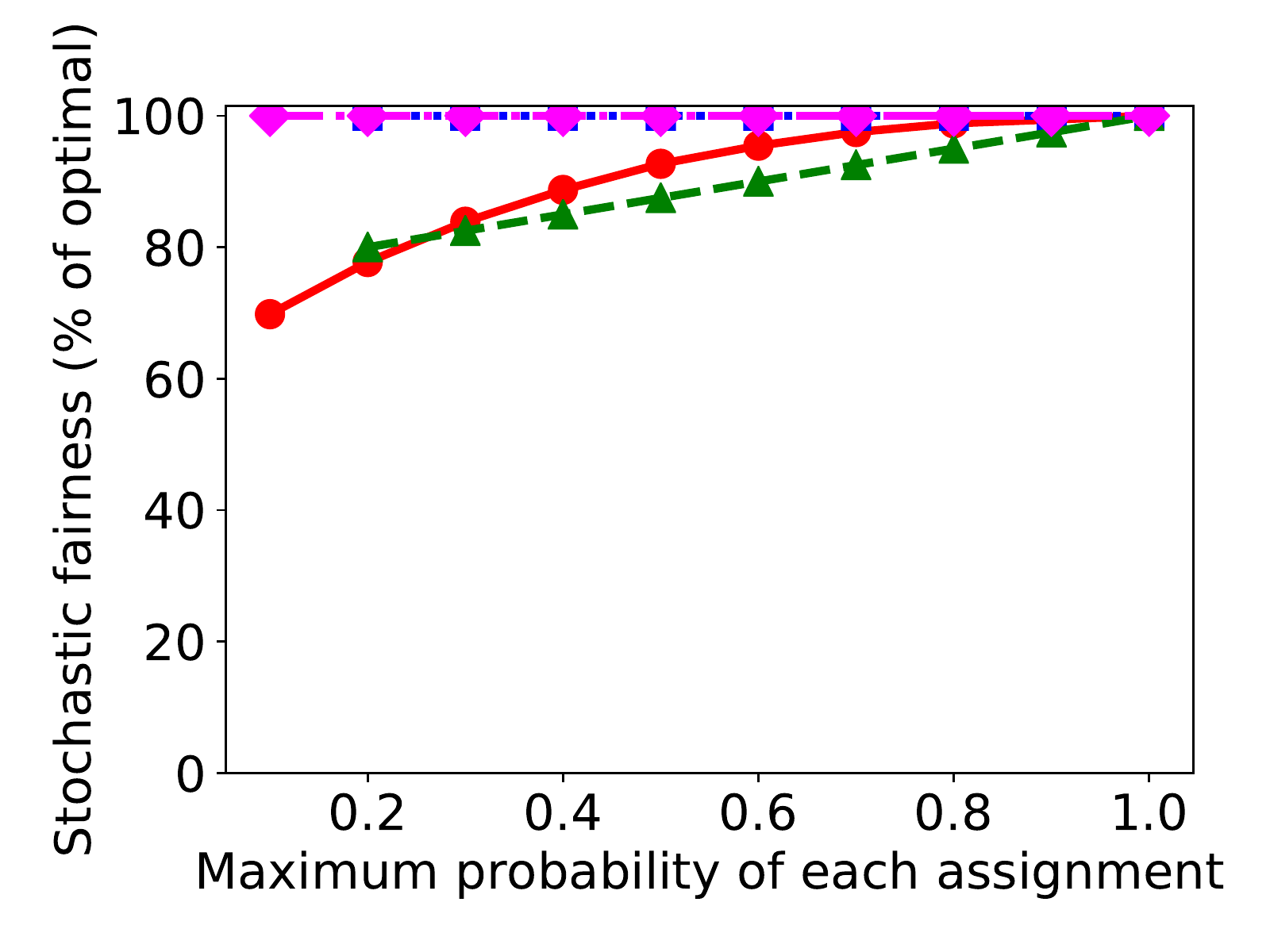} \end{subfigure}
    \caption{Experimental results for the Fair Pairwise-Constrained Problem.} \label{figfair}
\end{figure*}

In our setting of randomized assignments, we consider an analogous fairness objective, which we call the stochastic fairness objective: $\min_{\adpap \in \papset} \mathbb{E}\left[ \sum_{\adrev \in \revset} \simmat_{\adrev\adpap} \intas_{\adrev\adpap} \right]$. The problem involving this objective is defined as follows.
\begin{defn}[Fair Pairwise-Constrained Problem] The input to the problem is a similarity matrix $\simmat$ and a matrix $\maxprob \in [0, 1]^{\numrev \times \numpap}$. The goal is to find a randomized assignment of papers to reviewers that maximizes $\min_{\adpap \in \papset} \mathbb{E}\left[ \sum_{\adrev \in \revset} \simmat_{\adrev\adpap} \intas_{\adrev\adpap} \right]$ subject to the constraints that $\mathbb{P}[\intas_{\adrev\adpap} = 1] \leq \maxprob_{\adrev\adpap}, \forall \adrev \in \revset, \adpap \in \papset$.
\end{defn}
This problem definition is identical to that of the Pairwise-Constrained Problem (Definition~\ref{defnpairwise}), with the exception that the objective to maximize is now the stochastic fairness objective rather than the sum-similarity. Note that this objective is not equal to the ``expected fairness'' (i.e., $\mathbb{E}\left[ \min_{\adpap \in \papset} \sum_{\adrev \in \revset} \simmat_{\adrev\adpap} \intas_{\adrev\adpap} \right]$), but by Jensen's inequality it is an upper bound on the expected fairness.

Fortunately, this problem is solvable efficiently, as the following theorem states.
\begin{thm}
There exists an algorithm which returns an optimal solution to the Fair Pairwise-Constrained Problem in poly(\numrev, \numpap) time.
\end{thm}
We now present our algorithm for solving the Fair Pairwise-Constrained Problem, thereby proving the theorem. It proceeds in a similar manner as the algorithm for the Pairwise-Constrained Problem presented in Section~\ref{secbasicalgo}. 

The algorithm first finds an optimal fractional assignment matrix, since the stochastic fairness objective depends only on the marginal probabilities in the fractional assignment matrix. The optimal fractional assignment is found by the following LP, which we call $\mathcal{LP}3$:
\begin{align}
\argmax_{\fracas \in \mathbb{R}^{\numrev \times \numpap}, x \in \mathbb{R}} \quad & x \\
\text{subject to } \quad & 0 \leq \fracas_{\adrev\adpap} \leq 1 & \forall \adrev \in \revset, \forall \adpap \in \papset \\
&\sum_{\adpap \in \papset} \fracas_{\adrev\adpap} \leq \revload & \forall \adrev \in \revset \\
& \sum_{\adrev \in \revset} \fracas_{\adrev\adpap} = \papload & \forall \adpap \in \papset \\
& \fracas_{\adrev\adpap} \leq Q_{\adrev\adpap} & \forall \adrev \in \revset, \forall \adpap \in \papset \\
& x \leq \sum_{\adrev \in \revset} \simmat_{\adrev\adpap} \fracas_{\adrev\adpap} & \forall \adpap \in \papset.
\end{align}
For any $\fracas$, the optimal value of $x$ is always $\min_{\adpap \in \papset} \sum_{\adrev \in \revset} \simmat_{\adrev\adpap} \fracas_{\adrev\adpap}$, the stochastic fairness of $\fracas$. For a fixed $x$, the feasible region of $\fracas$ in $\mathcal{LP}3$ is exactly the space of fractional assignment matrices with stochastic fairness no less than $x$. Therefore, $\mathcal{LP}3$ will find an optimal fractional assignment matrix for the stochastic fairness objective.

Once an optimal fractional assignment matrix has been found, it only remains to sample a deterministic assignment from it. This is done with the sampling algorithm described in Section~\ref{secbasicsamplingalgo}, just as in the Pairwise-Constrained Problem.

We now present some empirical results for this algorithm on the four conference datasets described in Section~\ref{secexps}. We set all entries of $\maxprob$ equal to the same constant value $\maxprobconst$ (varied on the x-axis), and observe how the stochastic fairness objective of the assignment changes as $\maxprobconst$ increases from $0.1$ to $1$ with an interval of $0.1$. Since the expectation is inside a minimum in the objective, the objective cannot be estimated without bias by averaging together the stochastic fairness of sampled deterministic assignments. Due to this difficulty, we plot the exact objective of our randomized assignment (i.e., the optimal objective value of $\mathcal{LP}3$) rather than averaging over multiple samples, and report the objective as a percentage of the unconstrained optimal solution's objective (that is, the algorithm's solution when $\maxprobconst=1$). As Figure~\ref{figfair} shows, our algorithm finds a randomized assignment achieving $92.7\%$ of the optimal fairness objective on the ICLR dataset when $\maxprobconst = 0.5$.

\section{Bad-Assignment Probability Problem Variants} \label{badmatchapdx}
An input to both the Pairwise-Constrained Problem (Definition~\ref{defnpairwise}) and the Partition-Constrained Problem (Definition~\ref{defnpartition}) is the matrix $\maxprob$, where $\maxprob_{\adrev\adpap}$ denotes the maximum probability with which reviewer $\adrev$ should be assigned to paper $\adpap$. In practice, program chairs can set the values in this matrix based on their own beliefs about each reviewer-paper pair. However, it may be difficult for program chairs to translate their beliefs about the risk of assigning any reviewer-paper pair into appropriate values for $\maxprob$. In this appendix, we define alternate versions of these problems that allow the program chairs to codify their beliefs in a different way.

Define the assignment of reviewer $\adrev$ to paper $\adpap$ as ``bad'' if reviewer $\adrev$ intends to untruthfully review paper $\adpap$ (either because they intend to give a dishonest favorable review or because they intend to torpedo-review). Further define a matrix $\badprob \in [0, 1]^{\numrev \times \numpap}$ of bad-assignment probabilities, where $\badprob_{\adrev\adpap}$ represents the probability that the assignment of reviewer $\adrev$ to paper $\adpap$ would be a bad assignment; we assume that the events of each reviewer-paper assignment being bad are all independent of each other. The ``true value'' of $\badprob$ may not be known, but it can be set based on the program chairs' beliefs about the reviewers and authors or potentially estimated based on some data from prior conferences. The problem variants we present in the following subsections make use of these bad-assignment probabilities. 

We first consider the problem of limiting the probabilities of bad reviewer-paper assignments. We then consider the problem of limiting the probabilities that bad pairs of reviewers are assigned to the same paper.

\subsection{Handling Bad Reviewer-Paper Assignments}
We define an alternate version of the Pairwise-Constrained Problem using the bad-assignment probabilities:
\begin{defn}[Bad-Assignment Probability Pairwise-Constrained Problem]
The input to the problem is a similarity matrix $\simmat$, a matrix $\badprob \in [0, 1]^{\numrev \times \numpap}$ of bad-assignment probabilities, and a value $\maxbadprob \in [0, 1]$. The goal is to find a randomized assignment of papers to reviewers that maximizes $\mathbb{E}\left[\sum_{\adrev \in \revset} \sum_{\adpap \in \papset} \simmat_{\adrev\adpap} \intas_{\adrev\adpap}\right]$ subject to the constraints that $\badprob_{\adrev\adpap} \mathbb{P}[\intas_{\adrev\adpap} = 1] \leq \maxbadprob, \forall \adrev \in \revset, \adpap \in \papset$.
\end{defn}
$\badprob_{\adrev\adpap} \mathbb{P}[\intas_{\adrev\adpap} = 1]$ is exactly the probability that both (i) reviewer $\adrev$ is assigned to paper $\adpap$ and (ii) this assignment is bad, so the constraints in the problem limit this at $\maxbadprob$ for all $\adrev \in \revset$ and $\adpap \in \papset$. This version of the Pairwise-Constrained Problem may be useful in practice if program chairs find it easier to set the values of $\badprob$ than they would for $\maxprob$. 

We now show how to solve the Bad-Assignment Probability Pairwise-Constrained Problem, by translating it to the original Pairwise-Constrained Problem. Suppose that we have access to the matrix $\fracas$ of marginal assignment probabilities that occur under some randomized assignment. The randomized assignment obeys our constraints if and only if $\fracas_{\adrev\adpap} \badprob_{\adrev\adpap} \leq \maxbadprob, \forall \adrev \in \revset, \adpap \in \papset$. This observation leads to the following method of solving the Bad-Assignment Probability Pairwise-Constrained Problem:
\begin{itemize}
    \item Transform the given instance of the Bad-Assignment Probability Pairwise-Constrained Problem into an instance of the Pairwise-Constrained Problem by constructing a matrix of maximum probabilities $\maxprob$ where 
    \begin{align*}
    \maxprob_{\adrev\adpap} = \min\left\{ \maxbadprob / \badprob_{\adrev\adpap}, 1 \right\} \qquad \forall \adrev \in \revset, \adpap \in \papset.
    \end{align*} 
    \item Solve the Pairwise-Constrained Problem using the algorithm from Theorem~\ref{bvnthm}, described in Section~\ref{secbasicalgo}.
\end{itemize}


\subsection{Handling Bad Pairs of Reviewers}
Here, we first present an alternative version of the Partition-Constrained Problem and show how to solve it. We then present a different approach to handling the issue of bad reviewer pairs. 

\subsubsection{Constraints on Disjoint Reviewer Sets}
In the same way as done above for the Pairwise-Constrained Problem, we define an alternate version of the Partition-Constrained Problem:
\begin{defn}[Bad-Assignment Probability Partition-Constrained Problem]
The input to the problem is a similarity matrix $\simmat$, a matrix $\badprob \in [0, 1]^{\numrev \times \numpap}$ of bad-assignment probabilities, a value $\maxbadprob \in [0, 1]$, and a partition of the reviewer set into subsets $\adinst_1, \dots, \adinst_\numinst \subseteq \revset$. The goal is to find a randomized assignment of papers to reviewers that maximizes $\mathbb{E}\left[\sum_{\adrev \in \revset} \sum_{\adpap \in \papset} \simmat_{\adrev\adpap} \intas_{\adrev\adpap}\right]$ subject to the constraints that $\badprob_{\adrev\adpap} \mathbb{P}[\intas_{\adrev\adpap} = 1] \leq \maxbadprob, \forall \adrev \in \revset, \adpap \in \papset$ and $\mathbb{P}[\intas_{a\adpap} = 1 \land \intas_{b\adpap} = 1] = 0, \forall a, b \in \adinst_i, \forall i \in [\numinst]$.
\end{defn}
Just as for the Bad-Assignment Probability Pairwise-Constrained Problem, we solve this problem by first transforming an instance of this problem into an equivalent instance of the Partition-Constrained Problem, done by constructing a matrix of maximum probabilities $\maxprob$ where $\maxprob_{\adrev\adpap} = \min\left( \maxbadprob / \badprob_{\adrev\adpap}, 1 \right), \forall \adrev \in \revset, \adpap \in \papset$. We then solve this instance using the algorithm in Section~\ref{secpartitionproblem}.

\subsubsection{Constraints on the Expected Number of Bad Reviewers}
The Bad-Assignment Probability Partition-Constrained Problem requires a partition of the reviewer set and prevents pairs of reviewers from being assigned to the same paper if they are in the same subset of this partition. Alternatively, one may want to prevent pairs of reviewers from being assigned to the same paper based on whether $\badprob$ indicates that they are both likely to be bad assignments on this paper, rather than based on some partition of the reviewer set. In this way, we now present an alternative approach to handling the issue of bad reviewer pairs, which does not require a partition of the reviewer set. Rather than explicitly constraining the probabilities of certain same-subset reviewer-reviewer-paper triples as in the Bad-Assignment  Partition-Constrained Problem, we limit the \emph{expected} number of bad reviewers on each paper. 

The following problem states this goal:
\begin{defn}[Bad-Assignment Probability Expectation-Constrained Problem]
The input to the problem is a similarity matrix $\simmat$, a matrix $\badprob \in [0, 1]^{\numrev \times \numpap}$ of bad-assignment probabilities, a value $\maxbadprob \in [0, 1]$, and a value $\mu \in \mathbb{R}$. The goal is to find a randomized assignment of papers to reviewers that maximizes $\mathbb{E}\left[\sum_{\adrev \in \revset} \sum_{\adpap \in \papset} \simmat_{\adrev\adpap} \intas_{\adrev\adpap}\right]$ subject to the constraints that $\badprob_{\adrev\adpap} \mathbb{P}[\intas_{\adrev\adpap} = 1] \leq \maxbadprob, \forall \adrev \in \revset, \adpap \in \papset$ and $\sum_{\adrev \in \revset} W_{\adrev\adpap} \mathbb{E}[ M_{\adrev\adpap}] \leq \mu, \forall \adpap \in \papset$.
\end{defn}
We now present the algorithm that optimally solves this problem. The following LP, $\mathcal{LP}4$, finds a fractional assignment with expected number of bad reviewers on each paper no greater than $\mu$: 
\begin{align}
\argmax_{\fracas \in \mathbb{R}^{\numrev \times \numpap}} \quad & \sum_{\adpap \in \papset} \sum_{\adrev \in \revset} \simmat_{\adrev\adpap} \fracas_{\adrev\adpap} \\
\text{subject to } \quad & 0 \leq \fracas_{\adrev\adpap} \leq 1 & \forall \adrev \in \revset, \forall \adpap \in \papset \label{con:frac1} &\\
&\sum_{\adpap \in \papset} \fracas_{\adrev\adpap} \leq \revload & \forall \adrev \in \revset &\\
& \sum_{\adrev \in \revset} \fracas_{\adrev\adpap} = \papload & \forall \adpap \in \papset \label{con:frac3} &\\
& \fracas_{\adrev\adpap} \badprob_{\adrev\adpap} \leq \maxbadprob & \forall \adrev \in \revset, \forall \adpap \in \papset \label{con:badprob} &\\
& \sum_{\adrev \in \revset} \fracas_{\adrev\adpap} \badprob_{\adrev\adpap} \leq \mu & \forall \adpap \in \papset \label{con:badpair} &.
\end{align}
Constraints~(\ref{con:frac1}-\ref{con:frac3}) define the space of fractional assignment matrices, Constraint~(\ref{con:badprob}) ensures that the probability of each bad assignment occurring is limited at $\maxbadprob$, and Constraint~(\ref{con:badpair}) ensures that the expected number of bad reviewer-paper assignments for each paper is at most $\mu$. Therefore, $\mathcal{LP}4$ finds the optimal fractional assignment for the Bad-Assignment Probability Expectation-Constrained Problem. This fractional assignment can then be sampled from using the sampling algorithm in Section~\ref{secbasicsamplingalgo}. 

The above approach to controlling bad reviewer pairs is not directly comparable to the approach taken earlier when solving the Bad-Assignment Probability Partition-Constrained Problem. The Bad-Assignment Probability Expectation-Constrained Problem indirectly restricts pairs of reviewers from being assigned to the same paper based on whether $\badprob$ indicates that they are both likely to be bad assignments on that paper, instead of based on a partition of the reviewer set. This could be advantageous if the sets of likely-bad reviewers for each paper (as given by the probabilities in $\badprob$) are not expressed well by any partition of the reviewer set. However, handling suspicious reviewer pairs through constraining the expected number of bad reviewers per paper is weaker than directly constraining the probabilities of certain reviewer-reviewer-paper triples (as in the Bad-Assignment Probability Partition-Constrained Problem). First, it provides a guarantee only in expectation, and does not guarantee anything about the probabilities of the events we wish to avoid (that is, bad reviewer pairs being assigned to a paper). In addition, we here are assuming that the event of paper $\adpap$ and reviewer $\adrev$ being a bad assignment is independent of this event for all other reviewer-paper pairs; so, this method cannot address the issue of associations between reviewers, such as their presence at the same academic institution.

\section{Decomposition Algorithm for the Pairwise-Constrained Problem} \label{bvnthmproof}

In Section~\ref{secbasicalgo}, we provided the sampling algorithm that realizes Theorem~\ref{bvnthm}, thus solving the Pairwise-Constrained Problem (Definition~\ref{defnpairwise}). We here provide a decomposition algorithm to compute a full distribution over deterministic assignments for a given fractional assignment matrix (which the prior work \cite{Budish2009IMPLEMENTINGRA} does not).  
For simplicity, we assume here that all reviewer loads are met with equality (that is, $\sum_{\adpap \in \papset} \fracas_{\adrev\adpap} = \revload$ for all $\adrev \in \revset$); the extension to the case when reviewer loads are met with inequality is simple.

We first define certain concepts necessary for the algorithm. We then present a subroutine of the algorithm and prove its correctness. We then present the overall algorithm and prove its correctness. Finally, we analyze the time complexity of the algorithm.

\paragraph{Preliminaries:}
We define here three concepts used in the algorithm and its proof.

\begin{itemize}
\item A capacitated matching instance consists of a set of papers $\papset$, a set of reviewers $\revset$, and a capacity function $\capac: \papset \cup \revset \to \mathbb{Z}$. 
A solution to $(\papset, \revset, \capac)$ is a matrix $\fracas \in [0, 1]^{\numrev \times \numpap}$, where for any $\adpap \in \papset$,
\[
    \sum_{\adrev \in \revset} \fracas_{\adrev\adpap} = \capac(\adpap),
\]
and for any $\adrev \in \revset$,
\[
    \sum_{\adpap \in \papset} \fracas_{\adrev\adpap} = \capac(\adrev).
\]
The solution $\fracas$ is integral if $\fracas_{\adrev\adpap} \in \{0, 1\}$ for all $\adpap \in \papset$ and $\adrev \in \revset$.

\item For any $\revset$ and $\papset$, a maximum matching on a set $S \subseteq \revset \times \papset$ subject to capacities $\capac$ is a set $M \subseteq S$ such that $\sum_{\adrev \in \revset} \mathbb{I}[(\adrev, \adpap) \in M] \leq \capac(\adpap), \forall \adpap \in \papset$ and $\sum_{\adpap \in \papset}  \mathbb{I}[(\adrev, \adpap) \in M] \leq \capac(\adrev), \forall \adrev \in \revset$, and $|M|$ is maximized. 

\item For any $\revset$ and $\papset$, a perfect matching on a set $S \subseteq \revset \times \papset$ subject to capacities $\capac$ is a maximum matching on $S$ subject to $\capac$ that additionally satisfies $\sum_{\adrev \in \revset} \mathbb{I}[(\adrev, \adpap) \in M] = \capac(\adpap), \forall \adpap \in \papset$ and $\sum_{\adpap \in \papset}  \mathbb{I}[(\adrev, \adpap) \in M] = \capac(\adrev), \forall \adrev \in \revset$.
\end{itemize}

\paragraph{Decomposition subroutine:}
The following procedure, a subroutine of the overall algorithm, takes an instance $(\papset, \revset, \capac)$ and a solution to that instance $\fracas$ as input, and outputs an integral solution $\fracas_0$ to $(\papset, \revset, \capac)$ with weight $\alpha_0$  and a fractional solution $\fracas'$ to $(\papset, \revset, \capac)$ with strictly fewer fractional entries than $\fracas$. Moreover, $\fracas$, $\fracas_0$, $\alpha_0$, and $\fracas'$ satisfy $\fracas = \alpha_0 \fracas_0 + (1 - \alpha_0) \fracas'$.
\begin{enumerate}
    \item Let $E \subseteq \revset \times \papset$ be $E = \{ (\adrev, \adpap) \mid \fracas_{\adrev\adpap} \in (0, 1) \}$, and let $M_0 \subseteq \revset \times \papset$ be $M_0 = \{ (\adrev, \adpap) \mid \fracas_{\adrev\adpap} = 1 \}$.
    With this, define capacity function $\capac'$ as, for any $\adpap \in \papset$,
    \[
        \capac'(\adpap) = \capac(\adpap) - |\{(\adrev, \adpap) \mid \adrev \in \revset\} \cap M_0|
    \]
    and for any $\adrev \in \revset$,
    \[
        \capac'(\adrev) = \capac(\adrev) - |\{(\adrev, \adpap) \mid \adpap \in \papset\} \cap M_0|.
    \]
    \item Find a maximum matching $M \subseteq E$ on $E$ subject to capacity constraints $\capac'$.
    \item Set $\fracas_0$ as 
    \[
    (\fracas_0)_{\adrev\adpap} = \mathbb{I} \left[ (\adrev, \adpap) \in M \cup M_0 \right], \forall \adrev \in \revset, \adpap \in \papset.
    \]
    Set $\fracas'$ as
    \[
        \fracas'_{\adrev\adpap} = \frac{1}{(1 - \alpha_0)} (\fracas_{\adrev\adpap} - \alpha_0 ({\fracas_0})_{\adrev\adpap}),  \forall \adrev \in \revset, \adpap \in \papset.
    \]
        Set $\alpha_0 = \min( \{\fracas_{\adrev\adpap} \mid (\adrev, \adpap) \in M\} \cup \{1 - \fracas_{\adrev\adpap} \mid (\adrev, \adpap) \in E \setminus (M \cup M_0)\} )$.
\end{enumerate}

We prove the correctness of this subroutine in Lemma~\ref{decomplemma}. Before we do, we restate a result from prior work \cite{Budish2009IMPLEMENTINGRA} that we use in the proof, using our own notation. 

\begin{lemma}[{\cite[Thm. 1]{Budish2009IMPLEMENTINGRA}}] \label{bvnexistlemma}
    For any $(\papset, \revset, \capac)$ and any solution $\fracas$ to $(\papset, \revset, \capac)$, there exists some $z \in \mathbb{Z}$, integral solutions $\{\fracas_1, \dots, \fracas_z\}$ to $(\papset, \revset, \capac)$, and $\alpha$ lying on the $z$-dimensional simplex, such that $\fracas = \sum_{i = 1}^z \alpha_i \fracas_i$.
\end{lemma}

Now, the following lemma proves the correctness of the subroutine.

\begin{lemma} \label{decomplemma}
    The decomposition subroutine finds $\fracas_0$, $\alpha_0$, and $\fracas'$, such that (i) $\fracas_0$ is an integral solution to $(\papset, \revset, \capac)$, (ii) $\fracas'$ is a fractional solution to $(\papset, \revset, \capac)$, (iii) $\fracas'$ has strictly fewer fractional entries than $\fracas$, and (iv) $\fracas = \alpha_0 \fracas_0 + (1 - \alpha_0) \fracas'$.
\end{lemma}
\begin{proof}
    We first consider (i).
    The key step is to show that the maximum matching $M$ found in step 2 is a perfect matching with respect to $\capac'$, or equivalently, to show there is a perfect matching on $E$ with respect to $\capac'$. 
    Consider the capacitated matching instance $(\papset, \revset, \capac')$, and the solution $\fracas''$ where 
    \[
    \fracas''_{\adrev\adpap} = \begin{cases} 
    \fracas_{\adrev\adpap} & \text{if } \fracas_{\adrev\adpap} < 1 \\ 
    0 & \text{otherwise}.
    \end{cases}
    \]
    $\fracas''$ is a solution to $(\papset, \revset, \capac')$ by the construction of $\capac'$. By Lemma~\ref{bvnexistlemma}, $\fracas''$ is a convex combination of integral solutions to $(\papset, \revset, \capac')$.
    For some $z$, let $\{\fracas_1, \dots, \fracas_z\}$ and $\alpha$ be such a decomposition of $\fracas''$, where each $\fracas_i$ is an integral solution to $(\revset, \papset, \capac')$ and $\alpha_i$ is its associated weight.
    For each $i \in [z]$, let $M_i \subseteq \revset \times \papset$ be the set of $(\adrev, \adpap)$ pairs where $(\fracas_i)_{\adrev\adpap} = 1$.
    Since $\fracas_i$ is a solution to $(\revset, \papset, \capac')$, $M_i$ is a perfect matching with respect to $\capac'$. 
    By the definition of $\fracas''$, $(\adrev, \adpap) \in E$ if and only if $\fracas''_{\adrev\adpap} > 0$.
    Now since $\fracas'' = \sum_{i=1}^z \alpha_i \fracas_i$, $E = \bigcup_{i=1}^z M_i$.
    Since each $M_i$ is a perfect matching with respect to $\capac'$, $E$ contains a perfect matching with respect to $\capac'$ and so the maximum matching $M$ found is in fact a perfect matching with respect to $\capac'$.
    Therefore, $M \cup M_0$ is a perfect matching with respect to $\capac$ by the definition of $\capac'$. Therefore, $\fracas_0$ is an integral solution to $(\papset, \revset, \capac)$.
    
    For (ii), by the construction of $\fracas'$, all capacity constraints hold with equality. 
    We only need to show that $\fracas'_{\adrev\adpap} \in [0, 1]$ for any $(\adrev, \adpap)$.
    Consider any $(\adrev, \adpap)$.
    There are $3$ cases.
    If $(\adrev, \adpap) \in M_0$, then $\fracas'_{\adrev\adpap} = 1$.
    If $(\adrev, \adpap) \not\in M \cup M_0$, then the choice of $\alpha_0$ ensures that
    \[
        \fracas'_{\adrev\adpap} = \frac{1}{(1 - \alpha_0)} \fracas_{\adrev\adpap} \le \frac{1}{(1 - (1 - \fracas_{\adrev\adpap}))} \fracas_{\adrev\adpap} = 1
    \]
    and
    \[
    \fracas'_{\adrev\adpap} = \frac{1}{(1 - \alpha_0)} \fracas_{\adrev\adpap} \geq \fracas_{\adrev\adpap} \geq 0.
    \]
    If $(\adrev, \adpap) \in M$, the choice of $\alpha_0$ ensures that
    \[
        \fracas'_{\adrev\adpap} = \frac{1}{(1 - \alpha_0)} (\fracas_{\adrev\adpap} - \alpha_0) \ge \frac{1}{(1 - \alpha_0)} (\fracas_{\adrev\adpap} - \fracas_{\adrev\adpap}) = 0
    \]
    and
    \[
        \fracas'_{\adrev\adpap} = \frac{1}{(1 - \alpha_0)} (\fracas_{\adrev\adpap} - \alpha_0) \leq \fracas_{\adrev\adpap} \leq 1.
    \]
    As a result, $\fracas'$ is a solution to $(\papset, \revset, \capac)$.

    For (iii), the choice of $\alpha_0$ ensures that at least one of the inequalities above achieves equality.
    That is, there exists $(\adrev, \adpap)$ where $\fracas_{\adrev\adpap} \in (0, 1)$ such that $\fracas'_{\adrev\adpap} \in \{0, 1\}$.
    
    Finally, (iv) holds by the construction of $\fracas_0$ and $\fracas'$.
\end{proof}

\paragraph{Overall algorithm:}
Using the above subroutine, the overall algorithm proceeds in the following recursive way. It     takes as input a capacitated matching instance $(\papset, \revset, \capac)$ and a solution to that instance $\fracas$. It outputs integral solutions $\{\fracas_1, \dots, \fracas_z\}$ to $(\papset, \revset, \capac)$ and $\alpha$ lying on the $z$-dimensional simplex, such that $\fracas = \sum_{i = 1}^z \alpha_i \fracas_i$.
\begin{enumerate}
    \item 
    If $\fracas$ is integral, return solution $\{\fracas\}$ and weight $1$.
    \item Otherwise, decompose $\fracas$ into $\fracas_0$ (with weight $\alpha_0$) and $\fracas'$ using the above subroutine.
    \item Recursively call this algorithm with $(\papset, \revset, \capac)$ and $\fracas'$ as input, decomposing $\fracas'$ into solutions $\{\fracas_1, \dots, \fracas_z\}$ with weights $\alpha$.
    \item Define $\beta = (1-\alpha_0) \alpha$. Return the solutions $\{\fracas_0, \fracas_1, \dots, \fracas_z\}$ with weights $(\alpha_0, \beta_1, \dots, \beta_z)$.
\end{enumerate}

We now prove the correctness of this algorithm.
\begin{thm}
The decomposition algorithm correctly outputs integral solutions $\{\fracas_1, \dots, \fracas_z\}$ to $(\papset, \revset, \capac)$ and $\alpha$ lying on the $z$-dimensional simplex, such that $\fracas = \sum_{i = 1}^z \alpha_i \fracas_i$.
\end{thm}
\begin{proof}
We prove this statement by induction. If the algorithm returns in step 1, the theorem's statement holds. Now, assume that the theorem's statement holds for the decomposition returned by the recursive call to the algorithm in step 3, so that the following all hold: $\{\fracas_1, \dots, \fracas_z\}$ are integral solutions to $(\papset, \revset, \capac)$, $\alpha$ lies on the $z$-dimensional simplex, and $\fracas' = \sum_{i = 1}^z \alpha_i \fracas_i$. By Lemma~\ref{decomplemma}, $\fracas_0$ is an integral solution to $(\papset, \revset, \capac)$, so the $z+1$ solutions returned in step 4 are integral solutions to $(\papset, \revset, \capac)$. Since $\alpha_0 \in [0, 1]$, $\beta \in [0, 1]^z$, and $\alpha_0 + \sum_{i = 1}^z \beta_z = 1$, the weights returned in step 4 lie on the $z+1$ dimensional simplex. Finally, by Lemma~\ref{decomplemma},
\begin{align*}
\fracas &= \alpha_0 \fracas_0 + (1 - \alpha_0) \fracas' \\
&=\alpha_0 \fracas_0 + (1 - \alpha_0)  \sum_{i = 1}^z \alpha_i \fracas_i \\
&= \alpha_0 \fracas_0 + \sum_{i = 1}^z \beta_i \fracas_i.
\end{align*}
Therefore, the theorem's statement holds for the output of the algorithm in step 4. By induction, this proves the desired statement.
\end{proof}

This decomposition algorithm can be used as part of the algorithm that solves the Pairwise-Constrained Problem, substituting for the sampling algorithm described in Section~\ref{secbasicsamplingalgo}. It finds the full decomposition of the fractional assignment matrix into deterministic assignments rather than sampling a deterministic assignment. The capacity function $\capac$ used as the original input to the algorithm is defined as $\capac(\adrev) = \revload, \forall \adrev \in \revset$ and $\capac(\adpap) = \papload, \forall \adpap \in \papset$, and the input solution $\fracas$ is exactly the fractional assignment matrix found as the solution to $\mathcal{LP}1$. The output integral solutions represent deterministic assignments, and the corresponding weights represent the probability with which each assignment should be chosen.

\paragraph{Time complexity:}
Since $\fracas'$ has at least one fewer fractional entry than $\fracas$, the recursive procedure has depth $O(\numpap \numrev)$ and therefore makes $O(\numpap \numrev)$ calls to the decomposition subroutine. In each call, the bottleneck is finding a maximum matching on $E$ subject to capacities $\capac$. This can be solved as a max-flow problem on a graph with $O(\numpap + \numrev)$ vertices and $O(\numpap \numrev)$ edges \cite{cormen2009introduction}. Using Dinic's algorithm \cite{dinic1970algorithm}, the computation of each matching takes $O(\numpap \numrev (\numpap + \numrev)^2 )$ time, giving an overall time complexity of $O( \numpap^2 \numrev^2 (\numpap + \numrev)^2)$.

 \section{Proofs of Theorem~\ref{nph} and Corollary~\ref{nphcor}} \label{nphproof}
\paragraph{Proof of Theorem~\ref{nph}:}
We first define a decision variant of the Triplet-Constrained Problem, called ``Arbitrary-Constraint Feasibility.'' An instance of this problem is defined by the paper and reviewer loads $\papload$ and $\revload$, and a $3$-dimensional tensor $\maxprobtens \in [0, 1]^{\numrev \times \numrev \times \numpap}$. For all $i, j \in \revset, i \neq j$ and for all $\adpap \in \papset$, $\maxprobtens_{ij\adpap}$ denotes the maximum probability that both reviewers $i$ and $j$ are assigned to paper $\adpap$. The question is: does there exist a randomized assignment that obeys the constraints given by $\maxprobtens$? We next show that Arbitrary-Constraint Feasibility is NP-hard by a reduction from 3-Dimensional Matching.

An instance of 3-Dimensional Matching consists of three sets $X, Y, Z$ of size $s$, and a collection of tuples in $X \times Y \times Z$. It asks whether there exists a selection of $s$ tuples that includes each element of $X, Y,$ and $Z$ at most once. This problem is known to be NP-complete \cite{karp1972reducibility}.

Given such an instance of 3-Dimensional Matching, we construct an instance of Arbitrary-Constraint Feasibility. Set loads of $\papload = 2$ reviewers per paper and $\revload = 1$ paper per reviewer. Consider $|X|+|Y|$ reviewers (one for each element of $X \cup Y$) and $|Z|$ papers (one for each element of $Z$). Define the tensor $\maxprobtens$ to have $\maxprobtens_{ij\adpap}$ equal to $1$ if $(i, j, \adpap)$ is one of the tuples, and $0$ otherwise. 

We now show that a 3-Dimensional Matching instance is a yes instance (that is, the answer to it is ``yes'') if and only if the corresponding Arbitrary-Constraint Feasibility instance is a yes instance, thus proving that solving Arbitrary-Constraint Feasibility in polynomial time would allow us to solve 3-Dimensional Matching in polynomial time. If there exists a feasible reviewer-paper assignment in the corresponding Arbitrary-Constraint Feasibility instance, then we would answer yes for the original 3-Dimensional Matching instance; otherwise, if there does not exist a feasible reviewer-paper assignment, then we would answer no for the original 3-Dimensional Matching instance.

If the 3-Dimensional Matching instance is a yes (that is, there exists a valid selection of $s$ tuples), then consider the paper assignment that assigns the corresponding reviewers and paper within each triple in the matching. Each paper has exactly $2$ reviewers and each reviewer has exactly $1$ paper, so this is a deterministic assignment. Since it includes only the triples in the matching instance, it obeys the probability constraints of $\maxprobtens$, so the Arbitrary-Constraint Feasibility instance is a yes.

If the 3-Dimensional Matching instance is a no, then all choices of $s$ tuples include some element of $X, Y$, or $Z$ twice. If some element of $Z$ is chosen twice, then there must exist another element of $Z$ that is not included in any tuple. Therefore, any assignment of reviewer pairs to papers must either (a) include some reviewer-pair-to-paper assignment disallowed by $\maxprobtens$ (i.e., an assignment not in the collection of tuples), (b) make less than $s$ assignments of pairs to papers (and thus not assign to some paper), or (c) assign a reviewer twice or not assign some paper. So, no deterministic reviewer-paper assignment can meet the constraints of $\maxprobtens$. Now consider any randomized assignment, and select an arbitrary deterministic assignment in support of the randomized assignment. This deterministic assignment does not meet the constraints of $\maxprobtens$, so it must assign some reviewer $\adrev$ to some paper $\adpap$ that $\maxprobtens$ requires to have probability $0$. Therefore, since this deterministic assignment is in support, the randomized assignment assigns reviewer $\adrev$ to paper $\adpap$ with non-zero probability, thereby violating the constraints of $\maxprobtens$. Therefore, no randomized assignment can meet the constraints of $\maxprobtens$. Therefore, the Arbitrary-Constraint Feasibility instance is a no. This proves that Arbitrary-Constraint Feasibility is NP-hard.

Since even telling if the feasible region of randomized assignments is non-empty is NP-hard, optimizing any objective over this region is also NP-hard. Therefore, the Triplet-Constrained Problem is NP-hard.

\paragraph{Proof of Corollary~\ref{nphcor}:} 
Suppose that the polytope of implementable reviewer-reviewer-paper probabilities could be expressed in a polynomial number of linear inequality constraints (with the reviewer-reviewer-paper probabilities as variables). An LP could then be constructed with these inequalities as well as the inequalities given by a tensor $\maxprobtens$ of maximum reviewer-reviewer-paper probabilities. Solving this LP with any linear objective would then find a feasible point, solving Arbitrary-Constraint Feasibility. Since LPs can be solved in time polynomial in the number of variables and constraints, this is a contradiction unless $P \neq NP$.

\section{Proof of Lemma~\ref{samplinglemma} and Corollary~\ref{samplingcor}} \label{samplingthmproof}
In Section~\ref{secinstsamplingalgo}, we described the sampling algorithm that realizes Lemma~\ref{samplinglemma} and Corollary~\ref{samplingcor}. Here, we present proofs of these results.

\paragraph{Proof of Lemma~\ref{samplinglemma}:}
We first prove part (i) of the lemma. Consider any subset $\adinst$ and any paper $\adpap$, and recall that in Section~\ref{secinstsamplingalgo} we showed that the algorithm presented there has the property that the total load on each paper from each subset is preserved exactly if originally integral and rounded in either direction if originally fractional. If the total load from subset $\adinst$ on paper $\adpap$ is less than or equal to $1$ originally (i.e., $\sum_{\adrev \in \adinst} \fracas_{\adrev\adpap} \leq 1$), then this algorithm will only ever sample assignments with either $0$ or $1$ reviewers, so it never samples a integral assignment that assigns two reviewers from subset $\adinst$ to paper $\adpap$.

We now prove part (ii) of the lemma. Suppose that the total load from subset $\adinst$ on paper $\adpap$ is originally strictly greater than $1$ (i.e., $\sum_{\adrev \in \adinst} \fracas_{\adrev\adpap} > 1$). Let $X$ denote  a random variable that represents the number of reviewers from subset $\adinst$ on paper $\adpap$, that is, $X = \sum_{\adrev \in \adinst} \intas_{\adrev\adpap}$. Hence, we have $\mathbb{E}[X] = \sum_{\adrev \in \adinst} \fracas_{\adrev\adpap} > 1$. Suppose that we implement the marginal probabilities $\fracas$ as a distribution over deterministic assignments that places zero mass on any deterministic assignment where $X \geq 2$. Since $X$ is integral in any deterministic assignment, all of the mass must be placed on deterministic assignments where $X \leq 1$. Since $\mathbb{E}[X] > 1$, this is impossible. Therefore, $\fracas$ cannot be implemented without having some probability of placing two reviewers from subset $\adinst$ on paper $\adpap$, so the expected number of pairs of reviewers from subset $\adinst$ assigned to paper $\adpap$ must be non-zero for any sampling algorithm.

\paragraph{Proof of Corollary~\ref{samplingcor}:}
We now show that the distribution sampled from by the algorithm realizing Lemma~\ref{samplinglemma} minimizes the expected number of pairs of reviewers from each subset assigned to each paper. Consider any subset $\adinst$ and paper $\adpap$, and again let $X$ denote  a random variable that represents the number of reviewers from subset $\adinst$ on paper $\adpap$. The expected number of pairs of reviewers from subset $\adinst$ assigned to paper $\adpap$ is $\mathbb{E}\left[ {X \choose 2} \right] = \frac{1}{2} \mathbb{E}[ X^2 ] - \frac{1}{2} \mathbb{E}[X]$. Since $\mathbb{E}[X]$ is fixed for a given $\fracas$, we must only show that our chosen decomposition minimizes $\mathbb{E}[ X^2 ]$. 

Let $f$ be the probability mass function of $X$ under the distribution of $X$ produced by our sampling algorithm, so that $f(i) = P[X = i]$ for $i \in \{0, \dots, |\adinst|\}$. Let $f'$ be the probability mass function of $X$ under any different distribution produced by some sampling algorithm, so that $\exists i \in \{0, \dots, |\adinst|\}$ such that $f'(i) \neq f(i)$. Since both $f$ and $f'$ are produced by sampling algorithms, they must respect the marginal assignment probabilities given by $\fracas$.

First, assume that $\mathbb{E}[X] = \mu$ is integral. $\mathbb{E}[X] = \sum_{\adrev \in \adinst} \fracas_{\adrev\adpap}$, so $\mu$ is equal to the total load from subset $\adinst$ on paper $\adpap$. From Section~\ref{secinstsamplingalgo}, our sampling algorithm preserves exactly the loads from any subset on any paper that are originally integral, meaning that it will always assign exactly $\mu$ reviewers from subset $\adinst$ to paper $\adpap$. In other words, our sampling algorithm always gives the distribution of $X$ where $f(\mu) = 1$ and $f(i) = 0$ for $i \neq \mu$. Since all distributions of $X$ have the same expectation, $\sum_{i = 0}^{|\adinst|} f'(i) i = \mu$; we also know that $f'(i) > 0$ for some $i \neq \mu$. For this distribution, we have that
\begin{align*}
\mathbb{E}_{f'}[X^2] &= \sum_{i = 0}^{|\adinst|} f'(i) i^2 
= \sum_{\Delta = -\mu}^{|\adinst| - \mu} f'(\mu + \Delta) (\mu + \Delta)^2 
= \mu^2 + \sum_{\Delta = -\mu}^{|\adinst| - \mu} f'(\mu + \Delta) \Delta^2 
> \mu^2 = \mathbb{E}_{f}[X^2].
\end{align*}

Now, suppose that $\mathbb{E}[X] = \mu$ is not integral. From Section~\ref{secinstsamplingalgo}, our sampling algorithm rounds to a neighboring integer the loads from any subset on any paper that are originally not integral, meaning that it will always assign exactly $\muceil$ or $\mufloor$ reviewers from subset $\adinst$ to paper $\adpap$. In other words, our sampling algorithm only places probability mass on outcomes $X = \muceil$ or $X = \mufloor$, so $f(i) = 0$ for $i \not\in \{ \muceil, \mufloor \}$. There is only one way to do this so that $\mathbb{E}[X] = \mu$; exactly $f(\muceil) = \mu - \mufloor$ and $f(\mufloor) = \muceil - \mu$. Then under this distribution, via some algebraic simplifications,
\begin{align}
\mathbb{E}_{f}[X^2] &= f(\muceil) {\muceil}^2 + f(\mufloor) {\mufloor}^2 \nonumber \\
&=  -{\muceil}^2 + {\muceil} - \mu + 2{\muceil}\mu. \label{eqn:ourexp}
\end{align}
Under any other distribution of $X$ giving the probability mass function $f'$,
\begin{align}
&\mathbb{E}_{f'}[X^2] = \sum_{i = 0}^{|\adinst|} f'(i) i^2 \nonumber \\
&= \sum_{\Delta = -{\muceil}}^{|\adinst| - {\muceil}} \left( f'({\muceil} + \Delta) ({\muceil} + \Delta)^2 \right)  + 2 {\muceil} \sum_{\Delta = -{\muceil}}^{|\adinst| - {\muceil}} \left( f'({\muceil} + \Delta) \Delta \right)
+ \sum_{\Delta = -{\muceil}}^{|\adinst| - {\muceil}} \left( f'({\muceil} + \Delta) \Delta^2 \right) \nonumber \\
&= {\muceil}^2 + 2 {\muceil} (\mu - \muceil)
+ \sum_{\Delta = -{\muceil}}^{|\adinst| - {\muceil}} \left( f'({\muceil} + \Delta) \Delta^2\right).\label{eqn:otherexp}
\end{align}

We want to show that $\mathbb{E}_{f'}[X^2] > \mathbb{E}_{f}[X^2]$. From (\ref{eqn:ourexp}) and (\ref{eqn:otherexp}), it remains to show that
\begin{align*}
\sum_{i = 0}^{|\adinst|} f'(i) (i - {\muceil})^2 &> {\muceil} - \mu.
\end{align*}
Note that because $f'(i) \neq f(i)$ for some $i$, there exists some $j \not\in \{ \muceil, \mufloor \}$ such that $f'(j) > 0$. Further, $(i - {\muceil})^2 \geq ({\muceil} - i)$ for all integers $i$ and $(i - {\muceil})^2 > ({\muceil} - i)$ for all integers $i \not\in \{ \muceil, \mufloor \}$. Therefore,
\begin{align*}
\sum_{i = 0}^{|\adinst|} f'(i) (i - {\muceil})^2 &> \sum_{i = 0}^{|\adinst|} f'(i) ({\muceil} - i) = {\muceil} - \sum_{i = 0}^{|\adinst|} f'(i) i = {\muceil} - \mu.
\end{align*}
Therefore, $\mathbb{E}_{f'}[X^2] > \mathbb{E}_{f}[X^2]$ as desired, so $f$ is the probability mass function corresponding to the distribution of $X$ which minimizes $\mathbb{E}[X^2]$ (uniquely, since the inequality is strict). This concludes the proof that our algorithm minimizes $\mathbb{E}[X^2]$ and therefore minimizes the expected number of pairs from the same subset assigned to the same paper.

\section{Synthetic Simulations} \label{addexperiments}

\begin{figure*}[t!] 
    \centering
    \begin{subfigure}{0.75\textwidth}\includegraphics[width=1\textwidth]{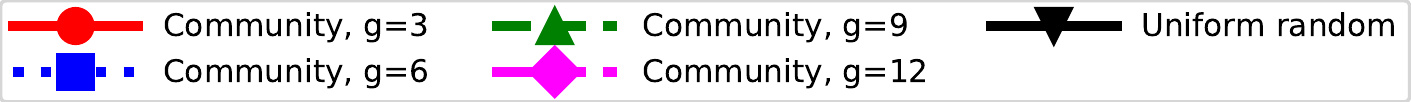}\label{figsimlegend} \end{subfigure} \\
    \begin{subfigure}{0.46\textwidth}\includegraphics[width=1\textwidth]{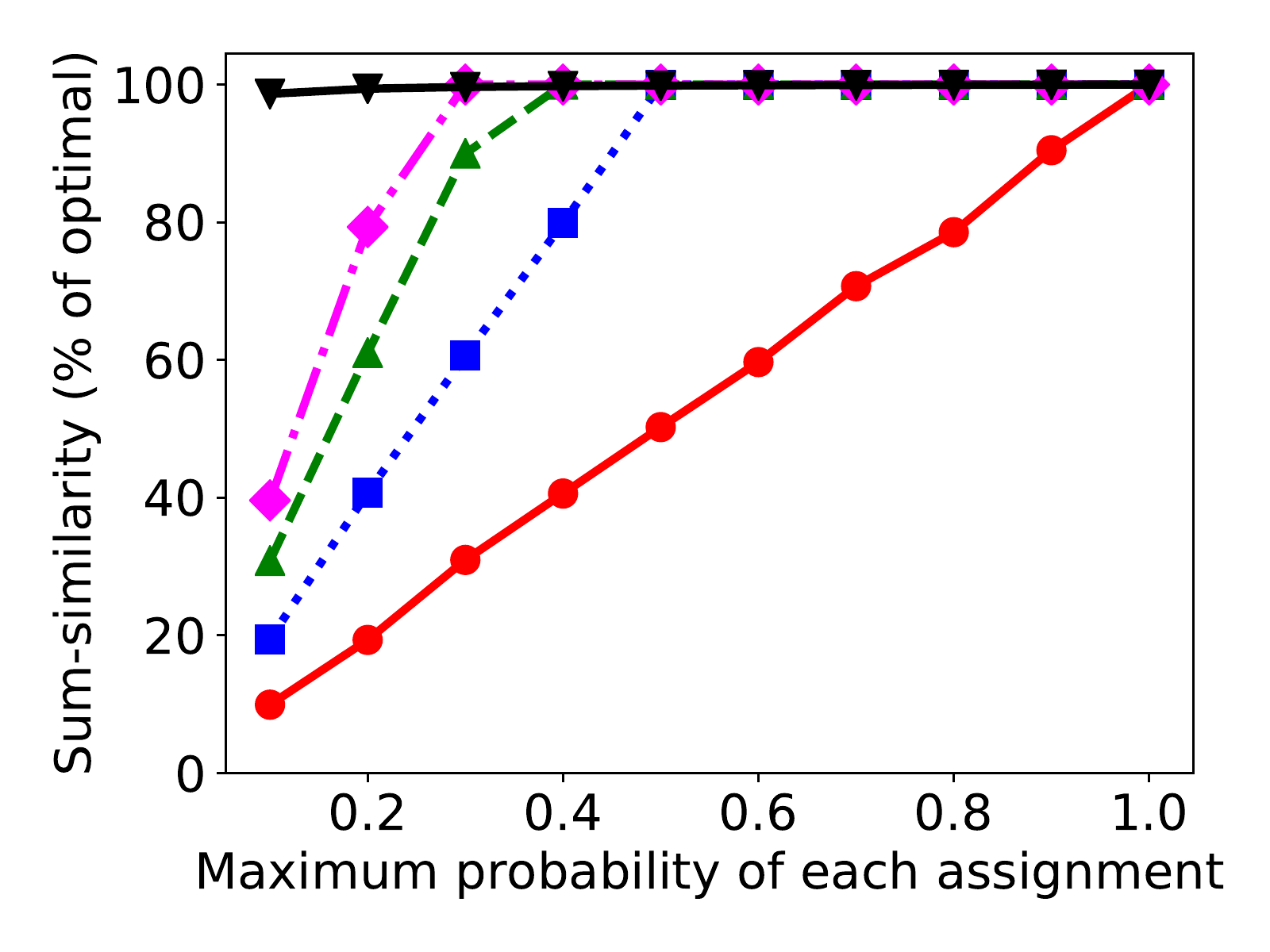}\caption{Pairwise-Constrained Problem}\label{figsimul} \end{subfigure}\hfill
    \begin{subfigure}{0.46\textwidth}\includegraphics[width=1\textwidth]{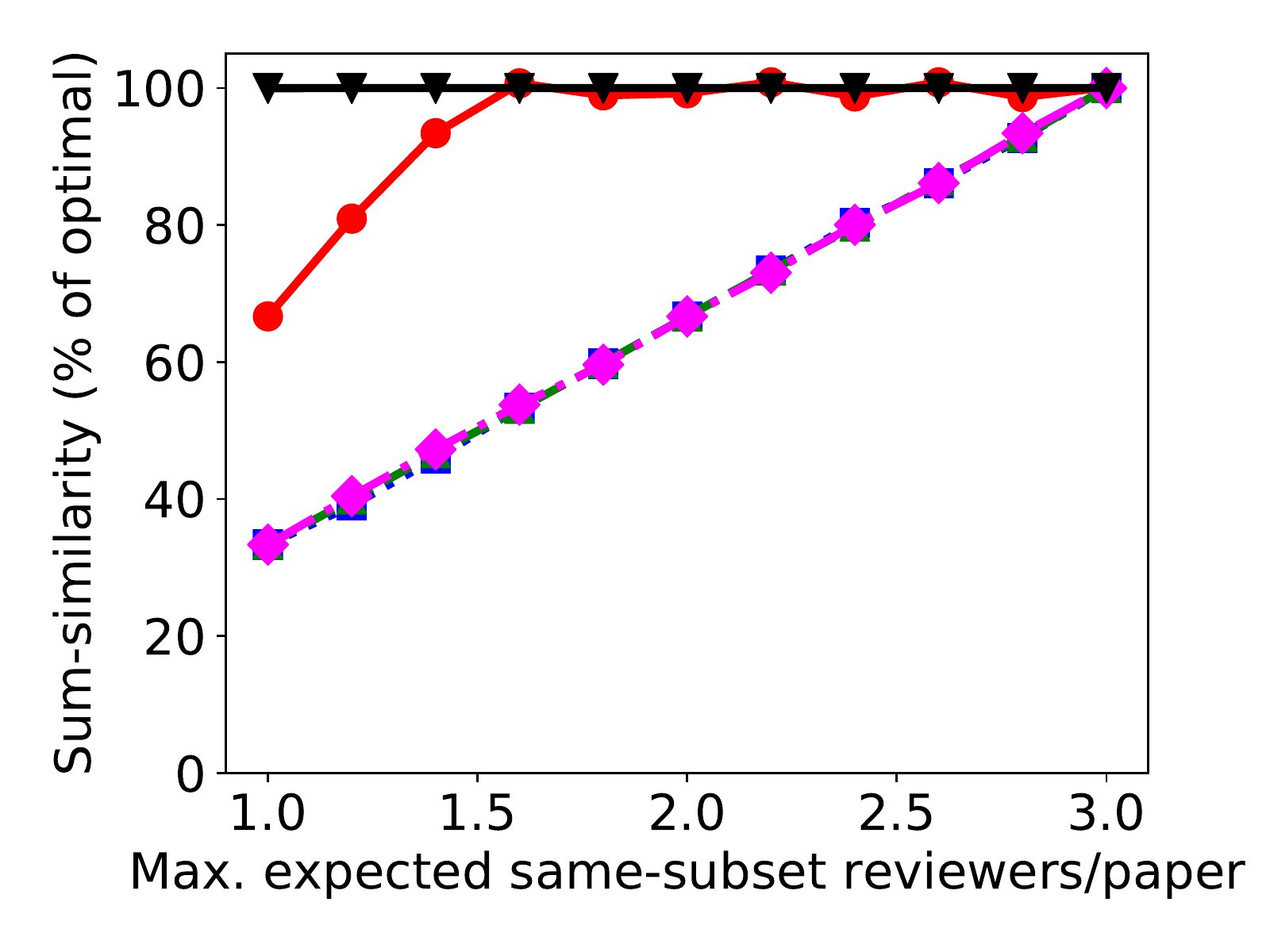}\caption{Partition-Constrained Problem}\label{figsimulinst} \end{subfigure} \\
    \begin{subfigure}{0.46\textwidth}\includegraphics[width=1\textwidth]{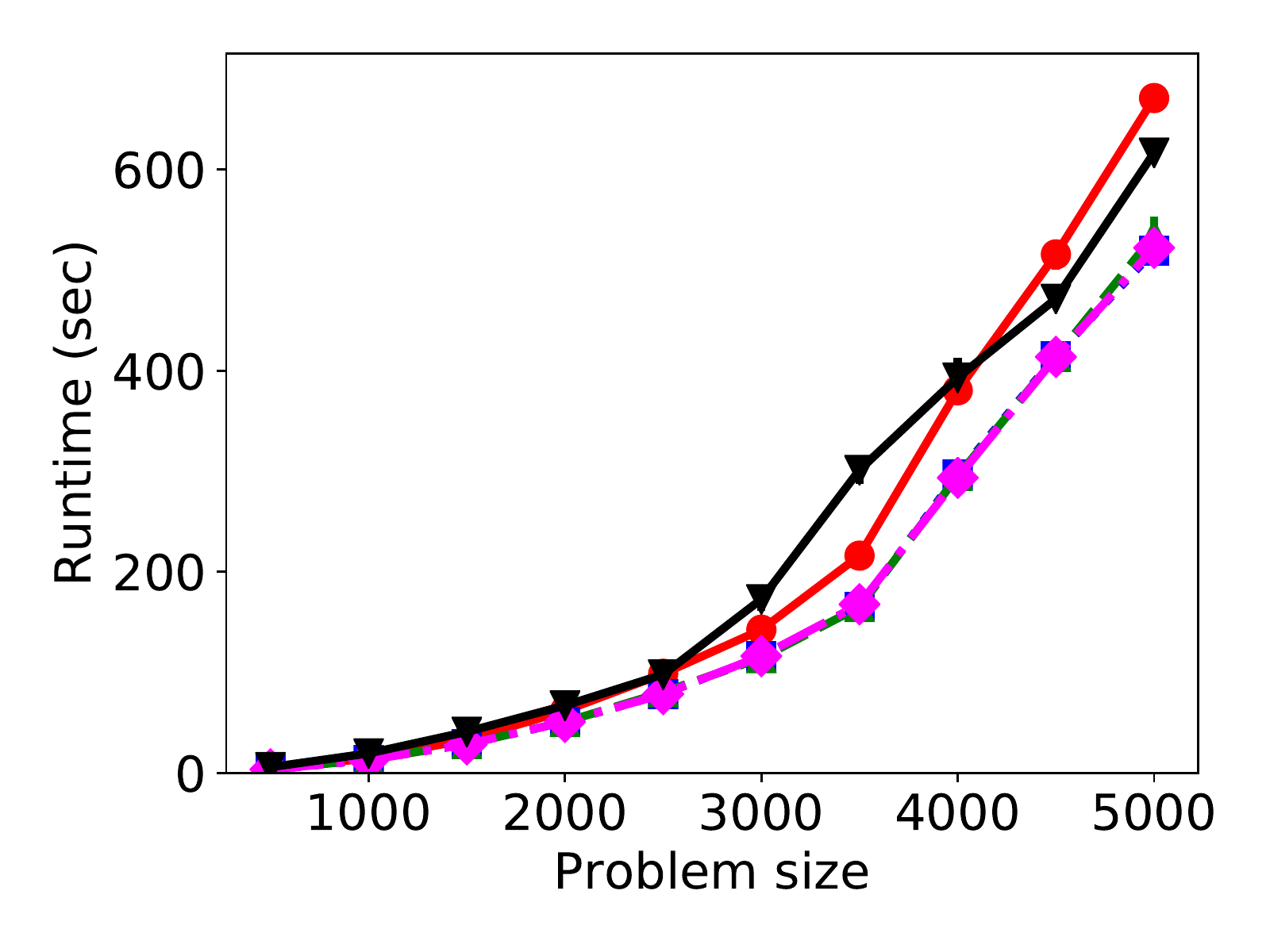}\caption{Runtime on Pairwise-Constrained Problem}\label{figruntime} \end{subfigure}
    \caption{Experimental results on synthetic simulations.} \label{figaddresults}
\end{figure*}

We now present experimental results on synthetic simulations. All results are averaged over $10$ trials with error bars plotted representing the standard error of the mean, although error bars are sometimes not visible since the variance is low. All experiments were run on a computer with $8$ cores and $16$ GB of RAM, running Ubuntu 18.04 and solving the LPs with Gurobi 9.0.2~\cite{gurobi}.

We consider two different simulations. First, we consider a simulated ``community model'' as used in past work \cite{fiez2020super}. In this model, $\numrev = \numpap =360$ and $\revload = \papload = 3$; it is further parameterized by a group size $\groupsize$. For all $i \in \{0, \groupsize, 2\groupsize, \dots, \numrev\}$, reviewers $i$ through $i + \groupsize -1$ have similarity $1$ with papers $i$ through $i+ \groupsize -1$ and similarity $0$ with all other papers. We consider four different group sizes $\groupsize$: 3, 6, 9, 12. We also consider a uniform random simulation, where each entry of the similarity matrix is independently and uniformly drawn from $[0, 1)$, fixing $\numrev = \numpap=1000$ and $\revload = \papload = 3$.

In Figure~\ref{figsimul}, we examine the performance of our algorithm for the Pairwise-Constrained Problem. For each simulation, we set all entries of $\maxprob$ to a constant $\maxprobconst$ and observe the sum-similarity as we vary $\maxprobconst$ (on the x-axis). The objective value is reported here as a percentage of the optimal unconstrained solution's objective, as was done in Section~\ref{secexps}. For the community models, the group size makes a large difference as to what an acceptable value of $\maxprobconst$ is. For example, with group size $6$ and $\maxprobconst = 0.5$, our algorithm will always assign all good reviewers to all papers; however, for any lower value of $\maxprobconst$ it can no longer do this and so the objective deteriorates rapidly. Note that since our algorithm is optimal, this deterioration is due to the problem being overconstrained for low values of $\maxprobconst$ and not due to an issue with the algorithm. For the uniform random simulation, our algorithm performs very well, since there are likely many reviewers with high similarity for each paper.

We also examine the performance of our algorithm for the Partition-Constrained Problem in Figure~\ref{figsimulinst}. For each simulation, we fix $\maxprobconst = 0.5$ and gradually loosen Constraint (\ref{cnt:inst}) in $\mathcal{LP}2$ by increasing the constant from $1$ to $3$ in increments of $0.2$, shown on the x-axis. We plot the sum-similarity objective of the resulting assignment, reported as a percentage of the optimal non-partition-constrained solution's objective (that is, the solution to the Pairwise-Constrained Problem with $\maxprobconst = 0.5$). For the community model simulations, we assign all reviewers in each group to the same subset of the partition. Since all of the reviewers who can review each paper well are in the same subset, this presents a highly constrained problem (which our algorithm is solving optimally). As expected, our algorithm trades off the number of same-subset reviewer pairs assigned to the same paper and the sum-similarity objective rather poorly (as would any other algorithm). Since $\maxprobconst = 0.5$, there is no difference between the cases with group size $6$ or greater. For the uniform random simulation, we assign random subsets of size $100$. Since there are likely many reviewers with high similarity for each paper in different subsets, our algorithm again performs very well.

In Figure~\ref{figruntime}, we show the runtime of our algorithm for the Pairwise-Constrained Problem on the various simulations, fixing $\maxprobconst = 0.5$ and varying $\numrev=\numpap$ on the x-axis. The runtime of our algorithm is similar across the different simulations. Our algorithm solves the uniform random simulation case with $\numrev = \numpap = 5000$ in just over $10$ minutes.

\end{document}